\documentclass{article}

\title{Equilibrium Propagation: Bridging the Gap Between Energy-Based Models and Backpropagation}
\author{
  Benjamin Scellier and Yoshua Bengio\footnote{Y.B. is also a Senior Fellow of CIFAR}\\
  Universit\'{e} de Montr\'{e}al, Montreal Institute for Learning Algorithms
}
\usepackage[T1]{fontenc}      % what fonts to use for printing characters (output)
\usepackage[utf8]{inputenc} % allows the user to input accented characters directly from the keyboard, alternatively can use \usepackage[latin1]{inputenc}
\usepackage[english]{babel}

\usepackage{amssymb,amsthm,amsfonts,amscd}
\usepackage{amsmath}
\usepackage{mathtools}   % extension of amsmath package
\usepackage{times}
\usepackage{textcomp}
\usepackage{fancyhdr}    % for the headers \usepackage{fancyhdr}, \usepackage{fancyheadings}
\usepackage{lastpage}
\usepackage{eurosym}     % for the euro symbol \euro{}
\usepackage{comment}
\newcommand   \E{\mathbb E}

\newcommand{\abs}[1]{\left\lvert #1\right \rvert}                       %Absolute value
\newcommand{\norm}[1]{\left\lVert #1\right \rVert}                      %Norm
\newcommand{\set}[1]{\left\{#1\right\}}                                 %Set
\newtheorem{prop}{Proposition}
\newtheorem{theo}[prop]{Theorem}
\newtheorem{lem}[prop]{Lemma}

\theoremstyle{remark}

\usepackage{listings}
\usepackage{color} % to define colors, use \usepackage{color} or \usepackage{xcolor}

\definecolor{mygreen}{rgb}{0,0.6,0}
\definecolor{mygray}{rgb}{0.5,0.5,0.5}
\definecolor{mymauve}{rgb}{0.58,0,0.82}

\usepackage[margin=3cm]{caption}
\usepackage{makeidx} % Index
\usepackage{graphicx}
\usepackage{algorithm}
\usepackage{algorithmic}
\usepackage{natbib}
\usepackage{caption}
\graphicspath{{.},{img/}}

\renewcommand \v{\mathrm v}
\newcommand \x{\mathrm x}
\newcommand \y{\mathrm y}

\makeindex

\usepackage{vmargin}
\setmarginsrb{2cm}{2cm}{2cm}{3cm}{0cm}{0.8cm}{0cm}{1cm}

\begin{document}

\maketitle

\abstract{
We introduce Equilibrium Propagation, a learning framework for energy-based models.
It involves only one kind of neural computation, performed
in both the first phase (when the prediction is made)
and the second phase of training (after the target or prediction error is revealed).
Although this algorithm computes the gradient of an objective function
just like Backpropagation, it does
not need a special computation or circuit for
the second phase, where errors are implicitly propagated.
Equilibrium Propagation shares similarities with Contrastive Hebbian Learning
and Contrastive Divergence while solving the theoretical issues of both algorithms:
our algorithm computes the gradient of a well defined objective function.
Because the objective function is defined in terms of local perturbations,
the second phase of Equilibrium Propagation corresponds to only nudging the prediction
(fixed point, or stationary distribution) towards a configuration that reduces prediction error.
In the case of a recurrent multi-layer supervised network,
the output units are slightly nudged towards their target in the second phase,
and the perturbation introduced at the output layer propagates backward in the hidden layers.
We show that the signal 'back-propagated' during this second phase
corresponds to the propagation of error derivatives and encodes the gradient of the objective function,
when the synaptic update corresponds to a standard form of spike-timing dependent plasticity.
This work makes it more plausible that a mechanism similar to Backpropagation could be implemented by brains,
since leaky integrator neural computation performs both inference and error back-propagation in our model.
The only local difference between the two phases is whether synaptic changes are allowed or not.
We also show experimentally that multi-layer recurrently connected networks with $1$, $2$ and $3$
hidden layers can be trained by Equilibrium Propagation on the permutation-invariant MNIST task.
}

%%%%%%%%%%%%%%%%%%%%%%%%%%%%%%%%%%%%%%%%%%%%%%%%%%%%%%%%%%%
%%%%%%%%%%%%%%%%%%%%%%% INTRODUCTION %%%%%%%%%%%%%%%%%%%%%%
%%%%%%%%%%%%%%%%%%%%%%%%%%%%%%%%%%%%%%%%%%%%%%%%%%%%%%%%%%%

\section{Introduction}

The Backpropagation algorithm to train neural networks is considered to be biologically implausible.
Among other reasons, one major reason is that Backpropagation requires a special computational circuit and a special kind of computation
in the second phase of training.
Here we introduce a new learning framework called Equilibrium Propagation, which requires only one computational circuit
and one type of computation for both phases of training.
Just like Backpropagation applies to any differentiable computational graph
(and not just a regular multi-layer neural network),
Equilibrium Propagation applies to a whole class of energy based models
(the prototype of which is the continuous Hopfield model).

In section \ref{sec:bio-backprop}, we revisit the continuous Hopfield model \citep{Hopfield84}
and introduce Equilibrium Propagation as a new framework to train it.
The model is driven by an energy function whose minima correspond to preferred states of the model.
At prediction time, inputs are clamped and
the network relaxes to a fixed point, corresponding to a local minimum of the energy function.
The prediction is then read out on the output units.
This corresponds to the first phase of the algorithm.
In the second phase of the training framework,
when the target values for output units are observed, the outputs are nudged
towards their targets and the network relaxes to a new but nearby
fixed point which corresponds to slightly smaller prediction error.
The learning rule, which is proved to perform gradient descent on the squared error,
is a kind of contrastive Hebbian learning rule
in which we {\em learn} (make more probable) the second-phase fixed point by reducing its energy
and {\em unlearn} (make less probable) the first-phase fixed point by increasing its energy.
However, our learning rule is not the usual contrastive Hebbian learning rule
and it also differs from Boltzmann machine learning rules, as discussed
in sections \ref{subsec:fn-vs-bm} and \ref{sec:diff-bml}.

During the second phase, the perturbation caused at the outputs propagates across hidden layers in the network.
Because the propagation goes from outputs backward in the network, it is better thought of as a 'back-propagation'.
It is shown by ~\citet{Bengio-arxiv2015,Bengio-et-al-NECO2017} that the early change of
neural activities in the second phase
corresponds to the propagation of error derivatives with respect to neural activities.
Our contribution in this paper is to go beyond the early change of neural activities
and to show that the second phase also implements the
(back)-propagation of error derivatives with respect to the synaptic weights, and
that this update corresponds to a form of spike-timing dependent plasticity,
using the results of~\citet{Bengio-et-al-NECO2017}.

In section \ref{sec:framework}, we present the general formulation of Equilibrium Propagation:
a new machine learning framework for energy-based models.
This framework applies to a whole class of energy based models,
which is not limited to the continuous Hopfield model but encompasses arbitrary
dynamics whose fixed points (or stationary distributions) correspond to minima
of an energy function.

In section \ref{sec:discussion}, we compare our algorithm to the existing learning algorithms for energy-based models.
The recurrent back-propagation algorithm introduced by \citet{Pineda87,Almeida87}
optimizes the same objective function as ours but it involves a different
kind of neural computation in the second phase of training, which is not satisfying from a biological perspective.
The contrastive Hebbian learning rule for continuous Hopfield nets described by \citet{Movellan-1990}
suffers from theoretical problems: learning may deteriorate
when the free phase and clamped phase land in different modes of the energy function.
The Contrastive Divergence algorithm \citep{Hinton2002} has theoretical issues too:
the ${\rm CD}_1$ update rule may cycle indefinitely \citep{sutskever2010convergence}.
The equivalence of back-propagation and contrastive Hebbian learning was shown by \citet{Xie+Seung-2003}
but at the cost of extra assumptions: their model requires
infinitesimal feedback weights and exponentially growing learning rates for remote layers.

Equilibrium Propagation solves all these theoretical issues at once.
Our algorithm computes the gradient of a sound objective function
that corresponds to local perturbations.
It can be realized with leaky integrator neural computation which
performs both {\em inference} (in the first phase)
and {\em back-propagation of error derivatives} (in the second phase).
Furthermore, we do not need extra hypotheses
such as those required by \citet{Xie+Seung-2003}.
Note that algorithms related to ours based on infinitesimal perturbations at the outputs
were also proposed by \citet{OReilly-1996,Hertz97}.

Finally, we show experimentally in section \ref{sec:experiments} that our model is trainable.
We train recurrent neural networks with 1, 2 and 3 hidden layers on MNIST and we achieve $0.00\%$ training error.
The generalization error lies between 2\% and 3\% depending on the architecture.
The code for the model is available\footnote{https://github.com/bscellier/Towards-a-Biologically-Plausible-Backprop}
for replicating and extending the experiments.

%%%%%%%%%%%%%%%%%%%%%%%%%%%%%%%%%%%%%%%%%%%%%%%%%%%%%%%%%%%%%%%%%%%%%%%%%%%%%%%%%%%%%%%%%%%
%%%%%%%%%%%%%%%% A MORE BIOLOGICALLY PLAUSIBLE BACKPROPAGATION ALGORITHM %%%%%%%%%%%%%%%%%%
%%%%%%%%%%%%%%%%%%%%%%%%%%%%%%%%%%%%%%%%%%%%%%%%%%%%%%%%%%%%%%%%%%%%%%%%%%%%%%%%%%%%%%%%%%%

\section{The Continuous Hopfield Model Revisited: Equilibrium Propagation as a More Biologically Plausible Backpropagation}
\label{sec:bio-backprop}

In this section we revisit the continuous Hopfield model \citep{Hopfield84}
and introduce Equilibrium Propagation, a novel learning algorithm to train it.
Equilibrium Propagation is similar in spirit to Backpropagation in the sense that
it involves the propagation of a signal from output units to input units backward in the network,
during the second phase of training.
Unlike Backpropagation, Equilibrium Propagation requires only one kind of neural computations for both phases of training,
making it more biologically plausible than Backpropagation.
However, several points still need to be elucidated from a biological perspective.
Perhaps the most important of them is that the model described in this section requires symmetric weights,
a question discussed at the end of this paper.

%%%%%%%%%%%%%%%%%%%%%%%%%%%%%%%%%%%%%%%%%%%%%%%%%%%%%%%%%%%%%
%%%%%%%%%%%%%%%% A KIND OF HOPFIELD ENERGY %%%%%%%%%%%%%%%%%%

\subsection{A Kind of Hopfield Energy}
\label{sec:supervised}

Previous work 
\citep{hinton1986learning,Friston+Stephan-2007,Berkes-et-al-2011} has hypothesized that,
given a state of sensory information, neurons are
collectively performing inference: they are moving towards configurations that
better 'explain' the observed sensory data. We can think of the neurons'
configuration as an
'explanation' (or 'interpretation') for the observed sensory data.
In the energy-based model presented here, that means
that the units of the network gradually move towards
lower energy configurations that are more probable, given the sensory input and
according to the current "model of the world" associated with the
parameters of the model.

We denote by $u$ the set of units of the network\footnote{For reasons of convenience, we use the same symbol $u$ to mean both the set of units and the value of those units}, and by $\theta=(W,b)$
the set of free parameters to be learned, which includes the synaptic weights $W_{ij}$ and the neuron biases $b_i$.
The units are continuous-valued and would correspond to averaged voltage potential across
time, spikes, and possibly neurons in the same minicolumn.
Finally, $\rho$ is a nonlinear activation function such that $\rho(u_i)$
represents the firing rate of unit $i$.

We consider the following energy function $E$, a kind of Hopfield energy,
first studied by \citet{Bengio-arxiv2015,Bengio-et-al-arxiv2015b,Bengio-et-al-arxiv2015,Bengio-et-al-NECO2017}:
\begin{equation}
	\label{eq:internal-potential}
	E(u) := \frac{1}{2} \sum_i u_i^2 - \frac{1}{2} \sum_{i \neq j} W_{ij} \rho(u_i) \rho(u_j) - \sum_i b_i \rho(u_i).
\end{equation}
Note that the network is recurrently connected with symmetric connections, that is $W_{ij} = W_{ji}$.
The algorithm presented here is applicable to any architecture (so long as connections are symmetric), even a fully connected network.
However, to make the connection to backpropagation more obvious, we will consider more specifically
a layered architecture with no skip-layer connections and no lateral connections within a layer (Figure \ref{fig:any_architecture}).

In the supervised setting studied here,
the units of the network are split in three sets: the inputs $\x$ which are always clamped
(just like in other models such as the conditional Boltzmann machine),
the hidden units $h$ (which may themselves be split in several layers) and the output units $y$.
We use the notation $\y$ for the targets, which should not be confused with the outputs $y$.
The set of all units in the network is $u=\set{\x,h,y}$.
As usual in the supervised learning scenario,
the output units $y$ aim to replicate their targets $\y$.
The discrepancy between the output units $y$ and the targets $\y$ is measured by the quadratic cost function
\begin{equation}
	\label{eq:external-potential}
	C := \frac{1}{2}\norm{y-\y}^2.
\end{equation}
The cost function $C$ also acts as an external potential energy for the output units,
which can drive them towards their target.
A novelty in our work, with respect to previously studied energy-based models,
is that we introduce the 'total energy function' $F$, which takes the form
\begin{equation}
	\label{eq:total-energy}
	F := E + \beta C,
\end{equation}
where $\beta \geq 0$ is a real-valued scalar that controls whether the output $y$
is pushed towards the target $\y$ or not, and by how much.
We call $\beta$ the 'influence parameter' or 'clamping factor'.
The total energy $F$ is the sum of two potential energies: the internal potential $E$ that models the interactions within the network,
and the external potential $\beta C$ that models how the targets influence the output units.
Contrary to Boltzmann Machines where the visible units are either free or (fully) clamped,
here the real-valued parameter $\beta$ allows the output units to be \textit{weakly clamped}.

%%%%%%%%%%%%%%%%%%%%%%%%%%%%%%%%%%%%%%%%%%%%%%%%%%%%%%%%%%%%%%%%%%%%%%%%%%%%%%%
%%%%%%%%%%%%%%%%%%%%%%%%%%% THE NEURONAL DYNAMICS %%%%%%%%%%%%%%%%%%%%%%%%%%%%%

\subsection{The Neuronal Dynamics}
\label{subsec:leaky}

\begin{figure}[ht]
	\begin{center}
		\centerline{
		\includegraphics[width=0.45\textwidth]{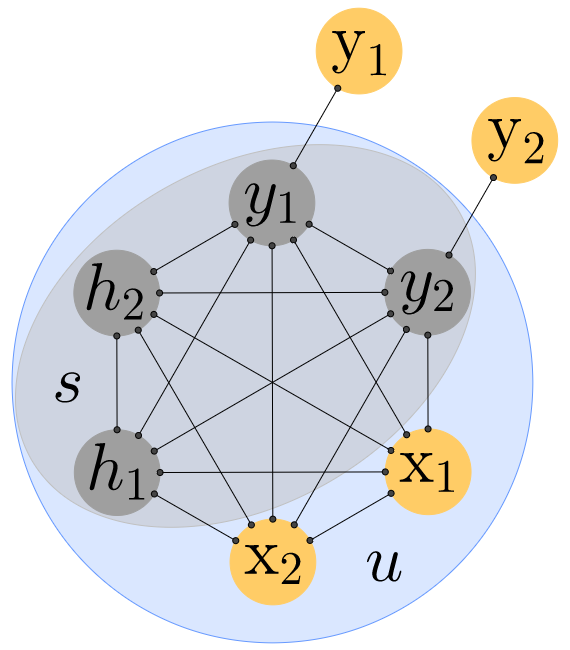}
		$\qquad$
		\includegraphics[width=0.45\textwidth]{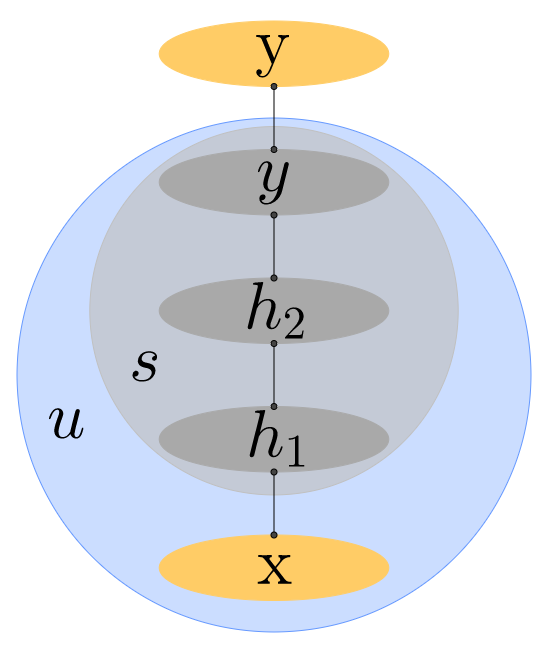}
		}
		\caption{The input units $\x$ are always clamped. The state variable $s$ includes the hidden units $h$ and output units $y$. The targets are denoted by $\y$.
		The network is recurrently connected with symmetric connections.
		\textbf{Left.} Equilibrium Propagation applies to any architecture, even a fully connected network. \textbf{Right.} The connection with Backpropagation is more obvious when the network has a layered architecture.}
		\label{fig:any_architecture}
	\end{center}
\end{figure}

We denote the state variable of the network by $s=\set{h,y}$ which does not include
the input units $\x$ since they are always clamped.
We assume that the time evolution of the state variable $s$ is governed by the gradient dynamics
\begin{equation}
	\label{eq:gradient-system}
   \frac{ds}{dt} = - \frac{\partial F}{\partial s}.
\end{equation}
Unlike more conventional artificial neural networks, the model studied here
is a continuous-time dynamical system described by the differential equation of motion Eq.~\ref{eq:gradient-system}.
The total energy of the system decreases as time progresses ($\theta$, $\beta$, $\x$ and $\y$ being fixed) since
\begin{equation}
  \frac{dF}{dt} = \frac{\partial F}{\partial s} \cdot \frac{ds}{dt} = - \norm{\frac{ds}{dt}}^2 \leq 0.
\end{equation}
The energy stops decreasing when the network has reached a fixed point:
\begin{equation}
	\label{eq:dE-ds=0}
  \frac{dF}{dt}=0 \qquad \Leftrightarrow \qquad \frac{ds}{dt}=0 \qquad \Leftrightarrow \qquad \frac{\partial F}{\partial s}=0.
\end{equation}

The differential equation of motion Eq.~\ref{eq:gradient-system} can be seen as a sum of two 'forces' that act on the temporal derivative of $s$:
\begin{equation}
   \frac{ds}{dt} = - \frac{\partial E}{\partial s} - \beta \frac{\partial C}{\partial s}.
\end{equation}
The 'internal force' induced by the internal potential (the Hopfield energy, Eq.~\ref{eq:internal-potential}) on the $i$-th unit is
\begin{equation}
	\label{eq:leaky-integrator}
  - \frac{\partial E}{\partial s_i} = \rho'(s_i) \left( \sum_{j \neq i} W_{ij} \rho(u_j) + b_i \right) - s_i,
\end{equation}
while the 'external force' induced by the external potential (Eq.~\ref{eq:external-potential}) on $h_i$ and $y_i$ is respectively
\begin{equation}
	\label{eq:external-force}
	- \beta \frac{\partial C}{\partial h_i} = 0 \qquad \text{and} \qquad
	- \beta \frac{\partial C}{\partial y_i} = \beta ( \y_i-y_i ).
\end{equation}

The form of Eq.~\ref{eq:leaky-integrator} is reminiscent of a leaky integrator neuron model,
in which neurons are seen as performing leaky temporal integration of their past inputs.
Note that the hypothesis of symmetric connections ($W_{ij} = W_{ji}$)
was used to derive Eq.~\ref{eq:leaky-integrator}.
As discussed in \cite{Bengio-arxiv2015}, the factor $\rho'(s_i)$ would suggest that
when a neuron is saturated (firing at the maximal or minimal
rate so that $\rho'(s_i) \approx 0$), its state is not sensitive to external inputs,
while the leaky term drives it out of the saturation regime,
towards its resting value $s_i=0$.

The form of Eq.~\ref{eq:external-force} suggests that
when $\beta=0$, the output units are not sensitive to the outside world $\y$.
In this case we say that the network is in the \textit{free phase} (or first phase).
On the contrary, when $\beta > 0$, the 'external force'
drives the output unit $y_i$ towards the target $\y_i$.
In this case, we say that the network is in the \textit{weakly clamped phase} (or second phase).

Finally, a more likely dynamics would include some form of noise.
The notion of fixed point is then replaced by that of stationary distribution.
In Appendix \ref{appendix:stochastic-framework}, we present a stochastic framework that naturally extends
the analysis presented here.

%%%%%%%%%%%%%%%%%%%%%%%%%%%%%%%%%%%%%%%%%%%%%%%%%%%%%%
%%%%%%%%%%%%%%%% BACKPROP %%%%%%%%%%%%%%%%%%%%%%%%%%%%

\subsection{Free Phase, Weakly Clamped Phase and Backpropagation of Errors}
\label{sec:backprop}

In the first phase of training, the inputs are clamped and $\beta=0$ (the output units are free).
We call this phase the \textit{free phase} and the state towards which the network converges is the \textit{free fixed point} $u^0$.
The prediction is read out on the output units $y$ at the fixed point.

In the second phase (which we call \textit{weakly clamped phase}), the influence parameter $\beta$ is changed to a small positive value $\beta>0$
and the novel term $\beta C$ added to the energy function (Eq.~\ref{eq:total-energy})
induces a new 'external force' that acts on the output units (Eq.~\ref{eq:external-force}).
This force models the observation of $\y$:
it nudges the output units from their
free fixed point value in the direction of their target.
Since this force only acts on the output units,
the hidden units are initially at equilibrium
at the beginning of the weakly clamped phase,
but the perturbation caused at the output units will propagate in the hidden units
as time progresses.
When the architecture is a multi-layered net
(Figure \ref{fig:any_architecture}. Right), the perturbation at the output layer
propagates backwards across the hidden layers of the network.
This propagation is thus better thought of as a 'back-propagation'.
The net eventually settles to a new fixed point (corresponding to the new positive value of $\beta$) which we call
\textit{weakly clamped fixed point} and denote by $u^\beta$.

Remarkably, the perturbation that is (back-)propagated during the second phase corresponds to the propagation of error derivatives.
It was first shown by~\citet{Bengio-arxiv2015} that, starting from the free fixed point,
the early changes of neural activities during the weakly clamped phase
(caused by the output units moving towards their target)
approximate some kind of error derivatives with respect to the layers' activities.
They considered a regular multi-layer neural network with no skip-layer connections
and no lateral connections within a layer.

In this paper, we show that the weakly clamped phase also implements the
(back)-propagation of error derivatives with respect to the synaptic weights.
In the limit $\beta \to 0$, the update rule
\begin{equation}
	\label{eq:chl}
	\Delta W_{ij} \propto \frac{1}{\beta} \left( \rho \left( u^\beta_i \right) \rho \left( u^\beta_j \right) - \rho \left( u^0_i \right) \rho \left( u^0_j \right) \right)
\end{equation}
gives rise to stochastic gradient descent on $J := \frac{1}{2}\norm{y^0-\y}^2$, where $y^0$ is the state of the output units at the free fixed point.
We will state and prove this theorem in a more general setting in section \ref{sec:framework}.
In particular, this result holds for any architecture
and not just a layered architecture (Figure \ref{fig:any_architecture})
like the one considered by ~\citet{Bengio-arxiv2015}.

The learning rule Eq.~\ref{eq:chl} is a kind of contrastive Hebbian learning rule,
somewhat similar to the one studied by \citet{Movellan-1990} and the Boltzmann machine learning rule.
The differences with these algorithms will be discussed in section \ref{sec:discussion}.

We call our learning algorithm Equilibrium Propagation.
In this algorithm, leaky integrator neural computation (as described in section \ref{subsec:leaky}),
performs both {\em inference} (in the free phase)
and {\em error back-propagation} (in the weakly clamped phase).

%%%%%%%%%%%%%%%%%%%%%%%%%%%%%%%%%%%%%%%%%%%%%%%%%%%%%%%%%%%%%%%%%%%%%%%%%%%%%%%
%%%%%%%%%%%%%%%% CONNECTION TO SPIKE TIMING DEPENDENT PLASTICITY %%%%%%%%%%%%%%

\subsection{Connection to Spike-Timing Dependent Plasticity}
\label{sec:stdp}

Spike-Timing Dependent Plasticity (STDP) is believed to be a prominent form of
synaptic change in neurons \citep{Markram+Sakmann-1995,Gerstner-et-al-1996}, and
see ~\citet{Markram-et-al-2012} for a review.

The STDP observations relate the expected change in synaptic
weights to the timing difference between postsynaptic spikes and
presynaptic spikes. This is the result of experimental observations
in biological neurons, but its role as part of a learning algorithm remains a topic
where more exploration is needed.
Here is an attempt in this direction.

Experimental results by \citet{Bengio-et-al-arxiv2015b} show that if the
temporal derivative of the synaptic weight $W_{ij}$ satisfies
\begin{equation}
	\label{eq:dWij=rhoi-dsj-dt}
	\frac{dW_{ij}}{dt} \propto \rho(u_i) \frac{d u_j}{dt},
\end{equation}
then one recovers the experimental observations by
\citet{Bi+Poo-2001} in biological neurons.
\citet{Xie+Seung-NIPS1999} have studied the learning rule
\begin{equation}
	\label{eq:dWij=rhoi-drhoj-dt}
	\frac{dW_{ij}}{dt} \propto \rho(u_i) \frac{d \rho(u_j)}{dt}.
\end{equation}
Note that the two rules Eq.~\ref{eq:dWij=rhoi-dsj-dt} and Eq.~\ref{eq:dWij=rhoi-drhoj-dt}
are the same up to a factor $\rho'(u_j)$.
An advantage of Eq.~\ref{eq:dWij=rhoi-drhoj-dt} is that
it leads to a more natural view of the update rule
in the case of tied weights studied here ($W_{ij}=W_{ji}$).
Indeed, the update should take into account the pressures from both the $i$ to $j$ and $j$ to $i$ synapses,
so that the total update under constraint is
\begin{equation}
	\label{eq:dWij=d-rhoi-rhoj}
	\frac{d W_{ij}}{dt} \propto \rho(u_i) \frac{d \rho(u_j)}{dt} + \rho(u_j) \frac{d \rho(u_i)}{dt} = \frac{d}{dt} \rho(u_i) \rho(u_j).
\end{equation}
By integrating Eq.~\ref{eq:dWij=d-rhoi-rhoj}
on the path from the free fixed point $u^0$ to the weakly clamped fixed point $u^\beta$
during the second phase, we get
\begin{equation}
	\Delta W_{ij} \propto \int \frac{d W_{ij}}{dt} dt
	= \int \frac{d}{dt} \rho(u_i) \rho(u_j) dt
	= \int d \left( \rho(u_i) \rho(u_j) \right)
	= \rho \left( u_i^\beta \right) \rho \left( u_j^\beta \right) - \rho \left( u_i^0 \right) \rho \left( u_j^0 \right),
\end{equation}
which is the same learning rule as Eq.~\ref{eq:chl} up to a factor $1/\beta$.
Therefore the update rule Eq.~\ref{eq:chl} can be interpreted as
a continuous-time integration of Eq.~\ref{eq:dWij=rhoi-drhoj-dt},
in the case of symmetric weights, on the path from $u^0$ to $u^\beta$ during the second phase.

\bigskip

We propose two possible interpretations for the synaptic plasticity in our model.

{\bf First view.}
In the first phase, a anti-Hebbian update occurs at the free fixed point
$\Delta W_{ij} \propto -\rho \left( u_i^0 \right) \rho \left( u_j^0 \right)$.
In the second phase, a Hebbian update occurs at the weakly-clamped fixed point
$\Delta W_{ij} \propto +\rho \left( u_i^\beta \right) \rho \left( u_j^\beta \right)$.

{\bf Second view.}
In the first phase, no synaptic update occurs: $\Delta W_{ij} = 0$.
In the second phase, when the neurons' state move from the free fixed point $u^0$ to the weakly-clamped fixed point $u^\beta$,
the synaptic weights follow the "tied version" of the continuous-time update rule
$\frac{d W_{ij}}{dt} \propto \frac{d}{dt} \rho(u_i) \frac{d \rho(u_j)}{dt} + \rho(u_j) \frac{d \rho(u_i)}{dt}$.

%%%%%%%%%%%%%%%%%%%%%%%%%%%%%%%%%%%%%%%%%%%%%%%%%%%%%%%%%%%%%%%%%%%%%%
%%%%%%%%% MACHINE LEARNING FRAMEWORK FOR ENERGY-BASED MODELS %%%%%%%%%
%%%%%%%%%%%%%%%%%%%%%%%%%%%%%%%%%%%%%%%%%%%%%%%%%%%%%%%%%%%%%%%%%%%%%%

\section{A Machine Learning Framework for Energy Based Models}
\label{sec:framework}

In this section we generalize the setting presented in section \ref{sec:bio-backprop}.
We lay down the basis for a new machine learning framework for energy-based models,
in which Equilibrium Propagation plays a role analog to Backpropagation in computational graphs
to compute the gradient of an objective function.
Just like the Multi Layer Perceptron is the prototype of computational graphs in which Backpropagation is applicable,
the continuous Hopfield model presented in section \ref{sec:bio-backprop} appears to be the prototype of models
which can be trained with Equilibrium Propagation.

In our new machine learning framework, the central object is the total energy function $F$:
all quantities of interest (fixed points, cost function, objective
function, gradient formula) can be defined or formulated directly in terms of $F$.

Besides, in our framework, the 'prediction' (or fixed point)
is defined \textit{implicitly} in terms of the data point and the parameters of the model,
rather than \textit{explicitly} (like in a computational graph).
This implicit definition makes applications on digital hardware (such as GPUs)
less practical as it needs long inference phases involving numerical
optimization of the energy function.
But we expect that this framework could be very efficient if
implemented by analog circuits, as suggested by \citet{Hertz97}.

The framework presented in this section is deterministic, but a natural extension
to the stochastic case is presented in Appendix \ref{appendix:stochastic-framework}.

%%%%%%%%%%%%%%%%%%%%%%%%%%%%%%%%%%%%%%%%%%%%%%%%%%%%%%%%%%%
%%%%%%%%%%%%%%%%%%%%%% TRAINING OBJECTIVE %%%%%%%%%%%%%%%%%

\subsection{Training Objective}
\label{sec:training-objective}

In this section, we present the general framework
while making sure to be consistent with the notations and terminology introduced in section \ref{sec:bio-backprop}.
We denote by $s$ the state variable of the network,
$\v$ the state of the external world (i.e. the data point being processed),
and $\theta$ the set of free parameters to be learned.
The variables $s$, $\v$ and $\theta$ are real-valued vectors.
The state variable $s$ spontaneously moves towards low-energy configurations of an energy function $E(\theta,\v,s)$.
Besides that, a cost function $C(\theta,\v,s)$ measures how 'good' is a state is.
The goal is to make low-energy configurations have low cost value.

For fixed $\theta$ and $\v$, we denote by $s_{\theta,\v}^0$ a local minimum of $E$,
also called fixed point, which corresponds to the 'prediction' from the model:
\begin{equation}
	s_{\theta,\v}^0 \in \underset{s}{\arg \min} \; E(\theta,\v,s).
\end{equation}
Here we use the notation ${\arg \min}$ to refer to the set of local minima.
The objective function that we want to optimize is
\begin{equation}
	\label{eq:objective-function}
	J(\theta,\v) := C \left( \theta,\v,s_{\theta,\v}^0 \right).
\end{equation}
Note the distinction between the cost function $C$ and the objective function $J$:
the cost function is defined for any state $s$,
whereas the objective function is the cost associated to the fixed point $s_{\theta,\v}^0$.
%(Figure \ref{fig:back-equi}, right)

Now that the objective function has been introduced, we define the training objective (for a single data point $\v$) as
\begin{equation}
	\text{find} \quad \underset{\theta}{\arg \min} \; J(\theta,\v).
\end{equation}
A formula to compute the gradient of $J$ will be given in section \ref{sec:equi-prop} (Theorem \ref{thm:deterministic}).
Equivalently, the training objective can be reformulated as the following constrained optimization problem:
\begin{align}
	\text{find}       & \quad \underset{\theta,s}{\arg \min} \; C(\theta,\v,s) \\
	\text{subject to} & \quad \frac{\partial E}{\partial s}(\theta,\v,s) = 0,
\end{align}
where the constraint $\frac{\partial E}{\partial s}(\theta,\v,s) = 0$ is the fixed point condition.
For completeness, a solution to this constrained optimization problem is given in Appendix \ref{appendix:constrained-optimization} as well.
Of course, both formulations of the training objective lead to the same gradient update on $\theta$.

Note that, since the cost $C(\theta,\v,s)$ may depend on $\theta$, it can include a regularization term of the form $\lambda \; \Omega \left( \theta \right)$,
where $\Omega \left( \theta \right)$ is a $L_1$ or $L_2$ norm penalty for example.

In section \ref{sec:bio-backprop}
we had $s=\set{h,y}$ for the state variable,
$\v=\set{\x,\y}$ for the state of the outside world,
$\theta=(W,b)$ for the set of learned parameters,
and the energy function $E$ and cost function $C$ were of the form
$E \left( \theta,\v,s \right) = E \left( \theta,\x,h,y \right)$ and
$C \left( \theta,\v,s \right) = C \left( y,\y \right)$.

%%%%%%%%%%%%%%%%%%%%%%%%%%%%%%%%%%%%%%%%%%%%%%%%%%%%%%%%%%%%%%%%%%%%%%%%%%%%%%
%%%%%%%%%%%%%%%%%%%%%%% TOTAL ENERGY FUNCTION %%%%%%%%%%%%%%%%%%%%%%%%%%%%%%%%

\subsection{Total Energy Function}
\label{sec:total-energy}

Following section \ref{sec:bio-backprop}, we introduce the total energy function
\begin{equation}
	\label{eq:tot-en}
	F(\theta,\v,\beta,s) := E(\theta,\v,s) + \beta \; C(\theta,\v,s),
\end{equation}
where $\beta$ is a real-valued scalar called 'influence parameter'.
Then we extend the notion of fixed point for any value of $\beta$.
The fixed point (or energy minimum), denoted by $s_{\theta,\v}^\beta$, is characterized by
\begin{equation}
	\label{eq:fx-pt-1}
  \frac{\partial F}{\partial s} \left( \theta,\v,\beta,s_{\theta,\v}^\beta \right) = 0
\end{equation}
and $\frac{\partial^2 F}{\partial s^2} \left( \theta,\v,\beta,s_{\theta,\v}^\beta \right)$ is a symmetric positive definite matrix.
Under mild regularity conditions on $F$, the implicit function theorem ensures that, for fixed $\v$,
the funtion $(\theta,\beta) \mapsto s_{\theta,\v}^\beta$ is differentiable.

%%%%%%%%%%%%%%%%%%%%%%%%%%%%%%%%%%%%%%%%%%%%%%%%%%%%%%%%%%%%%%%%%%%%%%%%%%%%%%
%%%%%%%%%%%%%%%%%%%%%%% THE LEARNING ALGORITHM %%%%%%%%%%%%%%%%%%%%%%%%%%%%%%%

\subsection{The Learning Algorithm: Equilibrium Propagation}
\label{sec:equi-prop}

\begin{theo}[Deterministic version]
	\label{thm:deterministic}
	The gradient of the objective function with respect to $\theta$ is given by the formula
	\begin{equation}
		\frac{\partial J}{\partial \theta}(\theta,\v) =
		\lim_{\beta \to 0} \frac{1}{\beta} \left( \frac{\partial F}{\partial \theta} \left( \theta,\v,\beta,s_{\theta,\v}^\beta \right) - \frac{\partial F}{\partial \theta} \left( \theta,\v,0,s_{\theta,\v}^0 \right) \right),
	\end{equation}
	or equivalently
	\begin{equation}
		\frac{\partial J}{\partial \theta}(\theta,\v) =
		\frac{\partial C}{\partial \theta} \left( \theta,\v,s_{\theta,\v}^0 \right) +
		\lim_{\beta \to 0} \frac{1}{\beta} \left( \frac{\partial E}{\partial \theta} \left( \theta,\v,s_{\theta,\v}^\beta \right) - \frac{\partial E}{\partial \theta} \left( \theta,\v,s_{\theta,\v}^0 \right) \right).
	\end{equation}
\end{theo}
Theorem \ref{thm:deterministic} will be proved in Appendix \ref{appendix:theorem}.
Note that the parameter $\beta$ in Theorem \ref{thm:deterministic} need not be positive (We only need $\beta \to 0$).
Using the terminology introduced in section \ref{sec:bio-backprop}, we call
$s_{\theta,\v}^0$ the free fixed point, and $s_{\theta,\v}^\beta$ the nudged fixed point
(or weakly-clamped fixed point in the case $\beta > 0$).
Moreover, we call a free phase (resp. nudged phase, or weakly-clamped phase)
a procedure that yields a free fixed point (resp. nudged fixed point, or weakly-clamped fixed point)
by minimizing the energy function $F$ with respect to $s$, for $\beta = 0$ (resp. $\beta \neq 0$).
Theorem \ref{thm:deterministic} suggests the following two-phase training procedure. Given a data point $\v$:
\begin{enumerate}
	\item Run a free phase until the system settles to a free fixed point $s_{\theta,\v}^0$ and collect $\frac{\partial F}{\partial \theta} \left( \theta,\v,0,s_{\theta,\v}^0 \right)$.
	\item Run a nudged phase for some $\beta \neq 0$ such that $\abs{\beta}$ is "small", until the system settles to a nudged fixed point $s_{\theta,\v}^\beta$ and collect $\frac{\partial F}{\partial \theta} \left( \theta,\v,\beta,s_{\theta,\v}^\beta \right)$.
	\item Update the parameter $\theta$ according to
		\begin{equation}
			\Delta \theta \propto - \frac{1}{\beta} \left( \frac{\partial F}{\partial \theta} \left( \theta,\v,\beta,s_{\theta,\v}^\beta \right) - \frac{\partial F}{\partial \theta} \left( \theta,\v,0,s_{\theta,\v}^0 \right) \right).
		\end{equation}
\end{enumerate}
Consider the case $\beta > 0$.
Starting from the free fixed point $s_{\theta,\v}^0$ (which corresponds to the 'prediction'),
a small change of the parameter $\beta$ (from the value $\beta=0$ to a value $\beta>0$) causes
slight modifications in the interactions in the network.
This small perturbation makes the network settle to a nearby weakly-clamped fixed point $s_{\theta,\v}^\beta$.
Simultaneously, a kind of contrastive update rule for $\theta$ is happening,
in which the energy of the free fixed point is increased
and the energy of the weakly-clamped fixed point is decreased.
We call this learning algorithm Equilibrium Propagation.

\bigskip

Note that in the setting introduced in section \ref{sec:supervised}
the total energy function (Eq.~\ref{eq:total-energy}) is such that
$\frac{\partial F}{\partial W_{ij}} = - \rho(u_i) \rho(u_j)$, in agreement with Eq.~\ref{eq:chl}.
In the weakly clamped phase, the novel term $\frac{1}{2}\beta \norm{y-\y}^2$ added to the energy $E$ (with $\beta>0$)
slightly attracts the output state $y$ to the target $\y$.
Clearly, the weakly clamped fixed point is better than the free fixed point in terms of prediction error.
The following proposition generalizes this property to the general setting.

\begin{prop}[Deterministic version]
	\label{prop:deterministic}
	The derivative of the function
	\begin{equation}
		\beta \mapsto C \left( \theta,\v,s_{\theta,\v}^\beta \right)
	\end{equation}
	at $\beta=0$ is non-positive.
\end{prop}
Proposition \ref{prop:deterministic} will also be proved in Appendix \ref{appendix:theorem}.
This proposition shows that, unless the free fixed point $s_{\theta,\v}^0$ is already optimal in terms of cost value,
for $\beta>0$ small enough, the weakly-clamped fixed point $s_{\theta,\v}^\beta$ achieves lower cost value than the free fixed point.
Thus, a small perturbation due to a small increment of $\beta$
would nudge the system towards a state that reduces the cost value.
This property sheds light on the update rule (Theorem \ref{thm:deterministic}), which can be seen as a kind of contrastive learning rule (somehow similar to the Boltzmann machine learning rule) where we {\em learn} (make more probable) the slightly better state $s_{\theta,\v}^\beta$ by reducing its energy
and {\em unlearn} (make less probable) the slightly worse state $s_{\theta,\v}^0$ by increasing its energy.

However, our learning rule is different from the Boltzmann machine learning rule
and the contrastive Hebbian learning rule.
The differences between these algorithms will be discussed in section \ref{sec:discussion}.

\begin{figure}[h!]
	\begin{center}
		\centerline{
		\includegraphics[width=0.4\columnwidth]{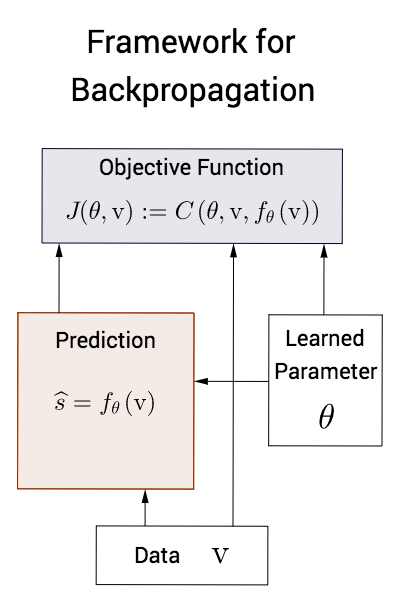}
		\includegraphics[width=0.4\columnwidth]{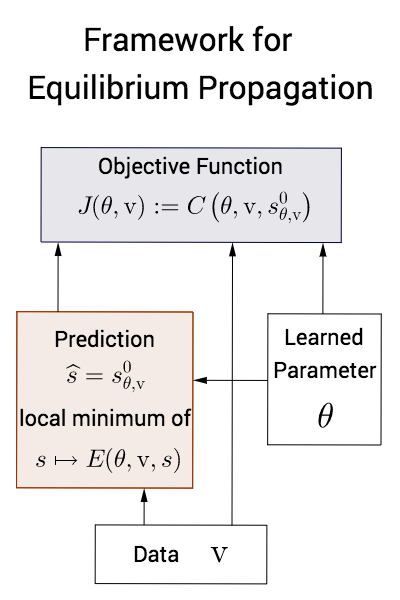}
		}
		\caption{Comparison between the traditional framework for Deep Learning and our framework.
		\textbf{Left.} In the traditional framework, the state of the network $f_\theta(\v)$ and the objective function $J(\theta,\v)$ are \textit{explicit} functions of $\theta$ and $\v$ and are computed \textit{analytically}. The gradient of the objective function is also computed analytically thanks to the Backpropagation algorithm (a.k.a automatic differentiation). \textbf{Right.} In our framework, the free fixed point $s_{\theta,\v}^0$ is an \textit{implicit} function of $\theta$ and $\v$ and is computed \textit{numerically}. The nudged fixed point $s_{\theta,\v}^\beta$ and the gradient of the objective function are also computed numerically, following our learning algorithm: Equilibrium Propagation.}
		\label{fig:back-equi}
	\end{center}
\end{figure}

%%%%%%%%%%%%%%%%%%%%%%%%%%%%%%%%%%%%%%%%%%%%%%%%%%%%%%%%%%%%%%%
%%%%%%%%%%%%%%% ANOTHER VIEW OF THE FRANEWORK %%%%%%%%%%%%%%%%%

\subsection{Another View of the Framework}

In sections \ref{sec:training-objective} and \ref{sec:total-energy} (as well as in section \ref{sec:bio-backprop})
we first defined the energy function $E$ and the cost function $C$,
and then we introduced the total energy $F := E + \beta C$.
Here we propose an alternative view of the framework, where we reverse the order in which things are defined.

Given a total energy function $F$ (which models all interactions within the network
as well as the action of the external world on the network), we can define all quantities of interest in terms of $F$.
Indeed, we can define the energy function $E$ and the cost function $C$ as
\begin{equation}
	\label{eq:E-C-F}
	E(\theta,\v,s) := F \left( \theta,\v,0,s \right) \qquad \text{and} \qquad
	C(\theta,\v,s) := \frac{\partial F}{\partial \beta} \left( \theta,\v,0,s \right),
\end{equation}
where $F$ and $\frac{\partial F}{\partial \beta}$ are evaluated with the argument $\beta$ set to $0$.
Obviously the fixed points $s_{\theta,\v}^0$ and $s_{\theta,\v}^\beta$ are directly defined in terms of $F$,
and so is the objective function \mbox{$J(\theta,\v) := C \left( \theta,\v,s_{\theta,\v}^0 \right)$.}
The learning algorithm (Theorem \ref{thm:deterministic}) is also formulated in terms of $F$.
\footnote{The proof presented in Appendix \ref{appendix:theorem} will show that $E$, $C$ and $F$ need not satisfy Eq.~\ref{eq:tot-en} but only Eq.~\ref{eq:E-C-F}.}
From this perspective, $F$ contains all the information about the model
and can be seen as the central object of the framework.
For instance, the cost $C$ represents the marginal variation of the total energy $F$ due to a change of $\beta$.

\bigskip

As a comparison, in the traditional framework for Deep Learning,
a model is represented by a (differentiable) computational graph
in which each node is defined as a function of its parents.
The set of functions that define the nodes fully specifies the model.
The last node of the computational graph represents the cost to be optimized,
while the other nodes represent the state of the layers of the network, as well as other intermediate computations.

In the framework for machine learning proposed here (the framework suited for Equilibrium Propagation),
the analog of the set of functions that define the nodes in the computational graph is the total energy function $F$.

%%%%%%%%%%%%%%%%%%%%%%%%%%%%%%%%%%%%%%%%%%%%%%%%%%%%%%%%%%%%%%%%%%%%%%%%%%%%%%
%%%%%%%%%%%%%%% BACKPROPAGATION VS EQUILIBRIUM PROPAGATION %%%%%%%%%%%%%%%%%%%

\subsection{Backpropagation Vs Equilibrium Propagation}

In the traditional framework for Deep Learning (Figure \ref{fig:back-equi}, left),
each node in the computational graph is an \textit{explicit} differentiable function of its parents.
The state of the network $\widehat{s} = f_\theta(\v)$
and the objective function $J(\theta,\v) = C \left( \theta, \v, f_\theta(\v) \right)$
are computed \textit{analytically}, as functions of $\theta$ and $\v$, in the forward pass.
The Backpropagation algorithm (a.k.a automatic differentiation) enables to compute the error derivatives analytically too,
in the backward pass.
Therefore, the state of the network $\widehat{s} = f_\theta(\v)$ (forward pass)
and the gradient of the objective function $\frac{\partial J}{\partial \theta}(\theta,\v)$ (backward pass)
can be computed {\em efficiently} and {\em exactly}.
\footnote{Here we are not considering numerical stability issues due to the encoding of real numbers with finite precision.}

In the framework for machine learning that we propose here (Figure \ref{fig:back-equi}, right),
the free fixed point $\widehat{s} = s_{\theta,\v}^0$ is an {\em implicit} function of
$\theta$ and $\v$, characterized by $\frac{\partial E}{\partial s}(\theta,\v,s_{\theta,\v}^0) = 0$.
The free fixed point is computed \textit{numerically}, in the free phase (first phase).
Similarly the nudged fixed point $s_{\theta,\v}^\beta$ is an implicit function of
$\theta$, $\v$ and $\beta$, and is computed numerically in the nudged phase (second phase).
Equilibrium Propagation \textit{estimates} (for the particular value of $\beta$ chosen in the second phase)
the gradient of the objective function $\frac{\partial J}{\partial \theta}(\theta,\v)$ based on these two fixed points.
The requirement for numerical optimization in the first and second phases make computations \textit{inefficient} and \textit{approximate}.
The experiments in section \ref{sec:experiments} will show that the free phase is fairly long when
performed with a discrete-time computer simulation.
However, we expect that the full potential of the proposed framework
could be exploited on analog hardware (instead of digital hardware), as suggested by \citet{Hertz97}.

%%%%%%%%%%%%%%%%%%%%%%%%%%%%%%%%%%%%%%%%%%%%%%%%%%%%%%%%%%%%%%%%
%%%%%%%%%%%%%%% RELATED WORK  %%%%%%%%%%%%%%%%%%%%%%%%%%%%%%%%%%
%%%%%%%%%%%%%%%%%%%%%%%%%%%%%%%%%%%%%%%%%%%%%%%%%%%%%%%%%%%%%%%%

\section{Related Work}
\label{sec:discussion}

In section \ref{sec:backprop} we have discussed the relationship between Equilibrium Propagation and Backpropagation.
In the weakly clamped phase, the change of the influence parameter $\beta$ creates a perturbation at the output layer which
propagates backwards in the hidden layers. The error derivatives and the gradient of the objective function are encoded
by this perturbation.

In this section we discuss the connection between our work and other algorihms, starting with Contrastive Hebbian Learning.
Equilibrium Propagation offers a new perspective on the relationship between Backpropagation in feedforward nets
and Contrastive Hebbian Learning in Hopfield nets and Boltzmann machines (Table \ref{table:correspondence}).

%\begin{minipage}{\textwidth}
\begin{center}
$\begin{array}{|c|c|c|c|c|}
\hline
	                    & \hbox{Backprop}        & \hbox{Equilibrium Prop}     & \hbox{Contrastive Hebbian Learning}      & \hbox{Almeida-Pineida}    \\
\hline
	\hbox{First Phase}  & \hbox{Forward Pass}    & \hbox{Free Phase}           & \hbox{Free Phase (or Negative Phase)}    & \hbox{Free Phase}         \\
	\hbox{Second Phase} & \hbox{Backward Pass}   & \hbox{Weakly Clamped Phase} & \hbox{Clamped Phase (or Positive Phase)} & \hbox{Recurrent Backprop} \\
\hline
\end{array}$
\captionof{table}{Correspondence of the phases for different learning algorithms: Back-propagation, Equilibrium Propagation (our algorithm), Contrastive Hebbian Learning (and Boltzmann Machine Learning) and Almeida-Pineida's Recurrent Back-Propagation}
\label{table:correspondence}
\end{center}
%\end{minipage}

%%%%%%%%%%%%%%%%%%%%%%%%%%%%%%%%%%%%%%%%%%%%%%%%%%%%%%%%%%%%%%%%
%%%%%%%%% LINK TO CONTRASTIVE HEBBIAN LERNING %%%%%%%%%%%%%%%%%%

\subsection{Link to Contrastive Hebbian Learning}
\label{subsec:fn-vs-bm}

Despite the similarity between our learning rule and the
Contrastive Hebbian Learning rule (CHL) for the continuous Hopfield model,
there are important differences.

First, recall that our learning rule is
\begin{equation}
	\Delta W_{ij} \propto \lim_{\beta \to 0} \frac{1}{\beta} \left( \rho \left( u^\beta_i \right) \rho \left( u^\beta_j \right) - \rho \left( u^0_i \right) \rho \left( u^0_j \right) \right),
\end{equation}
where $u^0$ is the free fixed point and $u^\beta$ is the \textit{weakly} clamped fixed point.
The Contrastive Hebbian Learning rule is
\begin{equation}
	\Delta W_{ij} \propto \rho \left( u^\infty_i \right) \rho \left( u^\infty_j \right) - \rho \left( u^0_i \right) \rho \left( u^0_j \right),
\end{equation}
where $u^\infty$ is the \textit{fully} clamped fixed point (i.e. fixed point with fully clamped outputs).
We choose the notation $u^\infty$ for the fully clamped fixed point
because it corresponds to $\beta \rightarrow +\infty$ with the notations of our model.
Indeed Eq.~\ref{eq:external-force} shows that
in the limit $\beta \rightarrow +\infty$, the output unit $y_i$ moves infinitely
fast towards $\y_i$, so $y_i$ is immediately clamped to
$\y_i$ and is no longer sensitive to the 'internal force' Eq.~\ref{eq:leaky-integrator}.
Another way to see it is by considering Eq.~\ref{eq:total-energy}:
as $\beta \rightarrow +\infty$, the only value of $y$ that gives finite energy is $\y$.

The objective functions that these two algorithms optimize also differ.
Recalling the form of the Hopfield energy (Eq.~\ref{eq:internal-potential})
and the cost function (Eq.~\ref{eq:external-potential}),
Equilibrium Propagation computes the gradient of
\begin{equation}
	J = \frac{1}{2} \norm{ y^0 - {\mathrm y} }^2,
\end{equation}
where $y^0$ is the output state at the free phase fixed point $u^0$,
while CHL computes the gradient of
\begin{equation}
	J_{\rm CHL} = E \left( u^\infty \right) - E \left( u^0 \right).
\end{equation}
The objective function for CHL has theoretical problems:
it may take negative values if the clamped phase and free phase stabilize in different modes of the energy function,
in which case the weight update is inconsistent and learning usually deteriorates, as pointed out by \citet{Movellan-1990}.
Our objective function does not suffer from this problem,
because it is defined in terms of local perturbations,
and the implicit function theorem guarantees that the weakly clamped fixed point
will be close to the free fixed point (thus in the same mode of the energy function).

We can also reformulate the learning rules and objective functions of these algorithms
using the notations of the general setting (section \ref{sec:framework}).
For Equilibrium Propagation we have
\begin{equation}
	\Delta \theta \propto - \lim_{\beta \to 0} \frac{1}{\beta} \left( \frac{\partial F}{\partial \theta} \left( \theta,\v,\beta,s_{\theta,\v}^\beta \right) - \frac{\partial F}{\partial \theta} \left( \theta,\v,0,s_{\theta,\v}^0 \right) \right) \qquad \text{and} \qquad
	J(\theta,\v) = \frac{\partial F}{\partial \beta} \left( \theta,\v,0,s_{\theta,\v}^0 \right).
\end{equation}
As for Contrastive Hebbian Learning, one has
\begin{equation}
	\Delta \theta \propto - \left( \frac{\partial F}{\partial \theta} \left( \theta,\v,\infty,s_{\theta,\v}^\infty \right) - \frac{\partial F}{\partial \theta} \left( \theta,\v,0,s_{\theta,\v}^0 \right) \right) \qquad \text{and} \qquad
	J_{\rm CHL}(\theta,\v) = F(\theta,\v,\infty,s_{\theta,\v}^\infty) - F(\theta,\v,0,s_{\theta,\v}^0),
\end{equation}
where $\beta = 0$ and $\beta=\infty$ are the values of $\beta$ corresponding to free and (fully) clamped outputs respectively.

Our learning algorithm is also more flexible because we are free to choose the cost function $C$ (as well as the energy funtion $E$),
whereas the contrastive function that CHL optimizes is fully determined by the energy function $E$.

%%%%%%%%%%%%%%%%%%%%%%%%%%%%%%%%%%%%%%%%%%%%%%%%%%%%%%%%%%%%%%%%
%%%%%%%%%%% LINK TO BOLTZMANN MACHINE LEARNING %%%%%%%%%%%%%%%%%

\subsection{Link to Boltzmann Machine Learning}
\label{sec:diff-bml}

Again, the log-likelihood
that the Boltzmann machine optimizes is determined by the Hopfield energy $E$,
whereas we have the freedom to choose the cost function in
the framework for Equilibrium Propagation.

As discussed in Section \ref{sec:backprop},
the second phase of Equilibrium Propagation (going from the free fixed point to the weakly clamped fixed point)
can be seen as a brief 'backpropagation phase' with weakly clamped target outputs.
By contrast, in the positive phase of the Boltzmann machine, the target is fully clamped,
so the (correct version of the) Boltzmann machine learning rule requires two separate and independent phases (Markov chains),
making an analogy with backprop less obvious.

Our algorithm is also similar in spirit to the CD algorithm
(Contrastive Divergence) for Boltzmann machines.
In our model, we start from a free fixed point
(which requires a long relaxation in the free phase)
and then we run a short weakly clamped phase.
In the CD algorithm, one starts from a positive equilibrium sample with the visible units clamped
(which requires a long positive phase Markov chain in the case of a general Boltzmann machine)
and then one runs a short negative phase.
But there is an important difference:
our algorithm computes the {\em correct}
gradient of our objective function (in the limit $\beta \to 0$),
whereas the CD algorithm computes a {\em biased estimator} of the
gradient of the log-likelihood.
The ${\rm CD}_1$ update rule is provably not the
gradient of any objective function
and may cycle indefinitely in some pathological cases \citep{sutskever2010convergence}.

Finally, in the supervised setting presented in Section \ref{sec:bio-backprop}, a more subtle difference with the Boltzmann machine
is that the 'output' state $y$ in our model
is best thought of as being part of the latent state variable $s$.
If we were to make an analogy with the Boltzmann machine,
the visible units of the Boltzmann machine would be $\v = \set{\x,\y}$,
while the hidden units would be $s=\set{h,y}$.
In the Boltzmann machine, the state of the external world is inferred directly on the visible units
(because it is a probabilistic generative model that maximizes the log-likelyhood of the data),
whereas in our model we make the choice to integrate in $s$ special latent variables $y$ that aim to match the target $\y$.

%%%%%%%%%%%%%%%%%%%%%%%%%%%%%%%%%%%%%%%%%%%%%%%%%%%%%%%%%%%%%%%%
%%%%%%%%%%% LINK TO RECURRENT BACK-PROPAGATION %%%%%%%%%%%%%%%%%

\subsection{Link to Recurrent Back-Propagation}

Directly connected to our model is the work by~\citet{Pineda87,Almeida87} on recurrent back-propagation.
They consider the same objective function as ours, but formulate the problem as a constrained optimization problem.
In Appendix \ref{appendix:constrained-optimization} we derive another proof for the learning rule (Theorem \ref{thm:deterministic})
with the Lagrangian formalism for constrained optimization problems.
The beginning of this proof is in essence the same as the one proposed by~\citet{Pineda87,Almeida87}, but
there is a major difference when it comes to solving Eq.~\ref{eq:lambda-star} for the costate variable $\lambda^*$.
The method proposed by~\citet{Pineda87,Almeida87} is to use Eq.~\ref{eq:lambda-star} to compute $\lambda^*$
by a fixed point iteration in a linearized form of the recurrent network.
The computation of $\lambda^*$ corresponds to their second phase, which they call {\em recurrent back-propagation}.
However, this second phase does not follow the same kind of dynamics as the first phase
(the free phase)
because it uses a linearization of the neural activation rather than the fully non-linear activation.
\footnote{Reccurent Back-propagation corresponds to Back-propagation Through Time (BPTT)
when the network converges and remains at the fixed point for a large number of time steps.}
From a biological plausibility point of view, having to use a different kind
of hardware and computation for the two phases is not satisfying.

By contrast, like the continuous Hopfield net and the Boltzmann machine,
our model involves only one kind of neural computations for both phases.

%%%%%%%%%%%%%%%%%%%%%%%%%%%%%%%%%%%%%%%%%%%%%%%%%%%%%%%%%%%%%%%%
%%%%%%%%%%% THE MODEL BY XIE & SEUNG %%%%%%%%%%%%%%%%%%%%%%%%%%%

\subsection{The Model by Xie \& Seung}

Previous work on the back-propagation interpretation of
contrastive Hebbian learning was done by~\citet{Xie+Seung-2003}.

The model by ~\citet{Xie+Seung-2003} is a modified version of the Hopfield model.
They consider the case of a layered MLP-like network, but their model can be extended to a more general connectivity, as shown here.
In essence, using the notations of our model (section \ref{sec:bio-backprop}), the energy function that they consider is
\begin{equation}
	E_{X\&S}(u) := \frac{1}{2} \sum_i \gamma^i u_i^2 - \sum_{i < j} \gamma^j W_{ij} \rho(u_i) \rho(u_j) - \sum_i \gamma^i b_i \rho(u_i).
\end{equation}
The difference with Eq.~\ref{eq:internal-potential} is that they introduce a parameter $\gamma$,
assumed to be small, that scales the strength of the connections.
Their update rule is the contrastive Hebbian learning rule which, for this particular energy function, takes the form
\begin{equation}
	\label{eq:chl-xie-seung}
	\Delta W_{ij} \propto - \left( \frac{\partial E_{X\&S}}{\partial W_{ij}} \left( u^\infty \right)
	- \frac{\partial E_{X\&S}}{\partial W_{ij}} \left( u^0 \right) \right)
	= \gamma^j \left( \rho \left( u^\infty_i \right) \rho \left( u^\infty_j \right) - \rho \left( u^0_i \right) \rho \left( u^0_j \right) \right)
\end{equation}
for every pair of indices $(i,j)$ such that $i < j$. Here $u^\infty$ and $u^0$ are the (fully) clamped fixed point and free fixed point respectively.
~\citet{Xie+Seung-2003} show that in the regime $\gamma \to 0$ this contrastive Hebbian learning rule is equivalent to back-propagation.
At the free fixed point $u^0$, one has $\frac{\partial E_{X\&S}}{\partial s_i}(u^0) = 0$ for every unit $s_i$
\footnote{Recall that in our notations, the state variable $s$ does not include the clamped inputs $\x$, whereas $u$ includes $\x$.},
which yields, after dividing by $\gamma^i$ and rearranging the terms
\begin{equation}
	s^0_i = \rho' \left( s^0_i \right) \left( \sum_{j<i} W_{ij} \rho \left( u^0_j \right) + \sum_{j>i} \gamma^{j-i} W_{ij} \rho \left( u^0_j \right) + b_i \right).
\end{equation}
In the limit $\gamma \to 0$, one gets $s^0_i \approx \rho'(s^0_i) \left( \sum_{j<i} W_{ij} \rho(u^0_j) + b_i \right)$,
so that the network almost behaves like a feedforward net in this regime.

As a comparison, recall that in our model (section \ref{sec:bio-backprop}) the energy function is
\begin{equation}
	E(u) := \frac{1}{2} \sum_i u_i^2 - \sum_{i < j} W_{ij} \rho(u_i) \rho(u_j) - \sum_i b_i \rho(u_i),
\end{equation}
the learning rule is
\begin{equation}
	\Delta W_{ij} \propto - \lim_{\beta \to 0} \frac{1}{\beta} \left( \frac{\partial E}{\partial W_{ij}} \left( u^\beta \right)
	- \frac{\partial E}{\partial W_{ij}} \left( u^0 \right) \right)
	= \lim_{\beta \to 0} \frac{1}{\beta} \left( \rho \left( u^\beta_i \right) \rho \left( u^\beta_j \right) - \rho \left( u^0_i \right) \rho \left( u^0_j \right) \right),
\end{equation}
and at the free fixed point, we have $\frac{\partial E}{\partial s_i}(u^0) = 0$ for every unit $s_i$, which gives
\begin{equation}
	s^0_i = \rho' \left( s^0_i \right) \left( \sum_{j \neq i} W_{ij} \rho \left( u^0_j \right) + b_i \right).
\end{equation}

Here are the main differences between our model and theirs.
In our model, the feedforward and feedback connections are both strong.
In their model, the feedback weights are tiny compared to the feedforward weights,
which makes the (recurrent) computations look almost feedforward.
In our second phase, the outputs are weakly clamped. In their second phase, they are fully clamped.
The theory of our model requires a unique learning rate for the weights, while in their model
the update rule for $W_{ij}$ (with $i<j$) is scaled by a factor $\gamma^j$ (see Eq.~\ref{eq:chl-xie-seung}).
Since $\gamma$ is small, the learning rates for the weights vary on many orders of magnitude in their model.
%making their model impractical.
Intuitively, these multiple learning rates are required to compensate for the small feedback weights.

%%%%%%%%%%%%%%%%%%%%%%%%%%%%%%%%%%%%%%%%%%%%%%%%%%%%%%%%%%%%%%%%
%%%%%%%%%%%%%%%%%%%%%%%%%%%%%%% IMPLEMENTATION %%%%%%%%%%%%%%%%%
%%%%%%%%%%%%%%%%%%%%%%%%%%%%%%%%%%%%%%%%%%%%%%%%%%%%%%%%%%%%%%%%

\section{Implementation of the Model and Experimental Results}
\label{sec:experiments}

In this section, we provide experimental evidence that our model described in section \ref{sec:bio-backprop} is trainable,
by testing it on the classification task of MNIST digits \citep{lecun1998mnist}.
The MNIST dataset of handwritten digits consists of 60,000 training examples and 10,000 test examples.
Each example $\x$ in the dataset is a gray-scale image of $28$ by $28$ pixels and comes with a label $\y \in \set{0,1,\ldots,9}$.
We use the same notation $\y$ for the one-hot encoding of the target, which is a 10-dimensional vector.

Recall that our model is a recurrently connected neural network with symmetric connections.
Here, we train multi-layered networks with 1, 2 and 3 hidden layers,
with no skip-layer connections and no lateral connections within layers.
Although we believe that analog hardware would be more suited for our model,
here we propose an implementation on digital hardware (a GPU).
We achieve 0.00\% training error.
The generalization error lies between 2\% and 3\% depending on the architecture (Figure \ref{fig:train_error}).

For each training example $(\x,\y)$ in the dataset, training proceeds as follows:
\begin{enumerate}
	\item Clamp $\x$.
	\item Run the free phase until the hidden and output units settle to the free fixed point, and collect $\rho \left( u_i^0 \right) \rho \left( u_j^0 \right)$ for every pair of units $i,j$.
	\item Run the weakly clamped phase with a "small" $\beta > 0$ until the hidden and output units settle to the weakly clamped fixed point, and collect $\rho \left( u_i^\beta \right) \rho \left( u_j^\beta \right)$.
	\item Update each synapse $W_{ij}$ according to
		\begin{equation}
			\Delta W_{ij} \propto \frac{1}{\beta} \left( \rho \left( u_i^\beta \right) \rho \left( u_j^\beta \right) - \rho \left( u_i^0 \right) \rho \left( u_j^0 \right) \right).
		\end{equation}
\end{enumerate}

The prediction is made at the free fixed point $u^0$
at the end of the first phase relaxation.
The predicted value $y_{\rm pred}$ is the index of the output unit
whose activation is maximal among the $10$ output units:
\begin{equation}
	y_{\rm pred} := \underset{i}{\arg \max} \; y_i^0.
\end{equation}
Note that no constraint is imposed on the activations
of the units of the output layer in our model,
unlike more traditional neural networks where a softmax output layer
is used to constrain them to sum up to $1$.
Recall that the objective function that we minimize is
the square of the difference between our prediction and the
one-hot encoding of the target value:
\begin{equation}
	J = \frac{1}{2} \norm{\y-y^0}^2.
\end{equation}

\begin{figure}[ht]
	\begin{center}
		\centerline{
			\includegraphics[width=0.5\columnwidth]{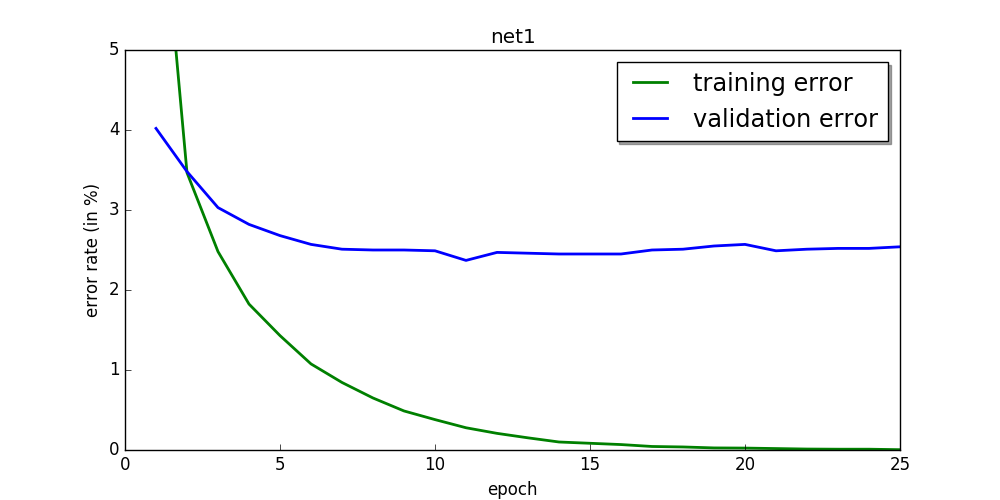}
			\includegraphics[width=0.5\columnwidth]{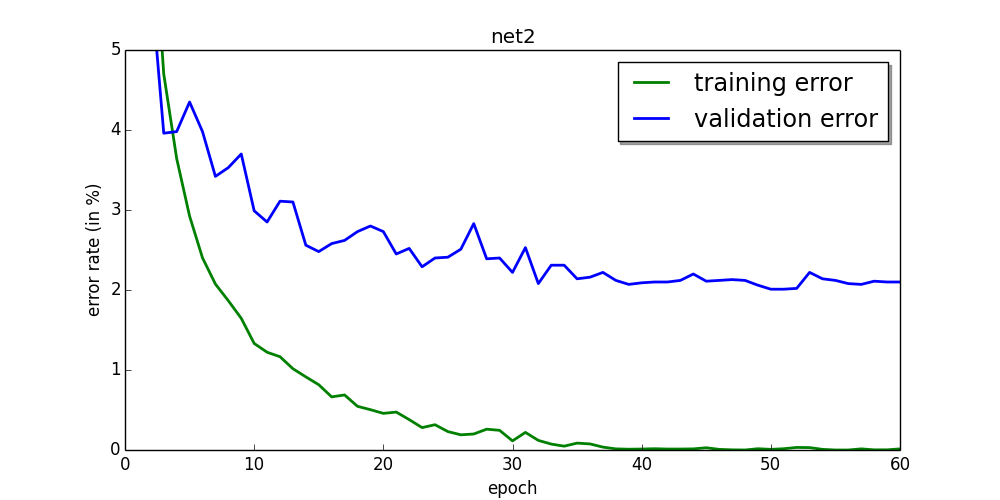}
		}
		\centerline{
		\includegraphics[width=\columnwidth]{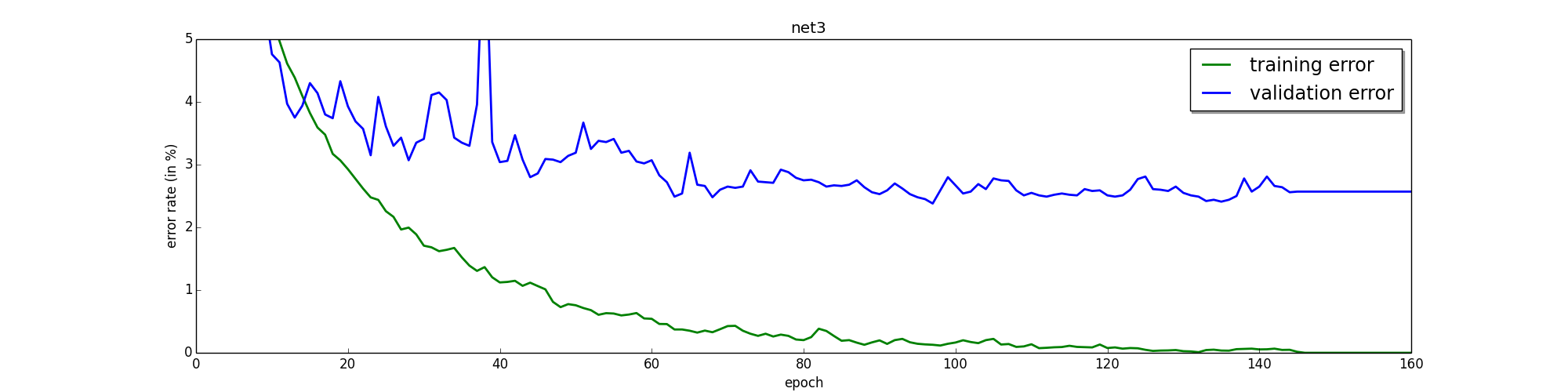}
		}
		\caption{Training and validation error for neural networks with 1 hidden layer of 500 units (top left), 2 hidden layers of 500 units (top right), and 3 hidden layers of 500 units (bottom). The training error eventually decreases to 0.00\% in all three cases.}
		\label{fig:train_error}
	\end{center}
\end{figure}

%%%%%%%%%%%%%%%%%%%%%%%%%%%%%%%%%%%%%%%%%%%%%%%%%%%%%%%%%%%%%%%%%%%%%%%%%%%%%%%
%%%%%%%%%%%%%%%%%%%%%%%%%%%%% FINITE DIFFERENCES %%%%%%%%%%%%%%%%%%%%%%%%%%%%%%

\subsection{Finite Differences}
\label{subsec:finite_diff}

{\bf Implementation of the differential equation of motion.}
First we clamp $\x$. Then the obvious way
to implement Eq.~\ref{eq:gradient-system} is to discretize time into short
time lapses of duration $\epsilon$ and to update each hidden and output unit $s_i$ according to
\begin{equation}
	\label{eq:gradient-descent}
	s_i \leftarrow s_i - \epsilon \frac{\partial F}{\partial s_i}(\theta,\v,\beta,s).
\end{equation}
This is simply one step of gradient descent on the total energy $F$, with step size $\epsilon$.

For our experiments, we choose the
hard sigmoid activation function $\rho(s_i) = 0 \vee s_i \wedge 1$,
where $\vee$ denotes the max and $\wedge$ the min.
For this choice of $\rho$, since $\rho'(s_i)=0$ for $s_i<0$,
it follows from Eq.~\ref{eq:leaky-integrator} and Eq.~\ref{eq:external-force}
that if $h_i<0$ then $\frac{\partial F}{\partial h_i}(\theta,\v,\beta,s) = - h_i > 0$.
This force prevents the hidden unit $h_i$ from going in the range of negative values.
The same is true for the output units.
Similarly, $s_i$ cannot reach values above $1$. As a consequence $s_i$
must remain in the domain $0 \leq s_i \leq 1$.
Therefore, rather than the
standard gradient descent (Eq.~\ref{eq:gradient-descent}), we will use a slightly
different update rule for the state variable $s$:
\begin{equation}
	\label{eq:clipped-gradient-descent}
	s_i \leftarrow 0 \vee \left( s_i - \epsilon \frac{\partial F}{\partial s_i}(\theta,\v,\beta,s) \right) \wedge 1.
\end{equation}
This little implementation detail turns out to be very important:
if the $i$-th hidden unit was in some state $h_i < 0$, then
Eq.~\ref{eq:gradient-descent} would give the update rule $h_i \leftarrow (1 - \epsilon) h_i$,
which would imply again $h_i < 0$ at the next time step (assuming $\epsilon < 1$).
As a consequence $h_i$ would remain in the negative range forever.

\bigskip

{\bf Choice of the step size $\epsilon$.} We find experimentally that the choice of
$\epsilon$ has little influence as long as $0 < \epsilon < 1$.
What matters more is the \textit{total duration of the relaxation}
$\Delta t = n_{\rm iter} \times \epsilon$
(where $n_{\rm iter}$ is the number of iterations).
In our experiments we choose $\epsilon = 0.5$
to keep $n_{\rm iter} = \Delta t / \epsilon$
as small as possible so as to
avoid extra unnecessary computations.

{\bf Duration of the free phase relaxation.}
We find experimentally that the number of iterations required in the
free phase to reach the free fixed point is large and grows fast as the number of
layers increases (Table \ref{table:hyperparameters}),
which considerably slows down training. More experimental
and theoretical investigation would be needed to analyze the number
of iterations required, but we leave that for future work.

{\bf Duration of the weakly clamped phase.}
During the weakly clamped phase,
we observe that the relaxation to the weakly clamped fixed point is not necessary.
We only need to 'initiate' the movement of the units, and for that we use the following heuristic.
Notice that the time constant of the integration process
in the leaky integrator equation Eq.~\ref{eq:leaky-integrator} is $\tau=1$.
This time constant represents the time
needed for a signal to propagate from a layer to the next one with
"significant amplitude". So the time
needed for the error signals to back-propagate in the network is $N \tau = N$,
where $N$ is the number of layers (hiddens and output) of the network.
Thus, we choose to perform $N/\epsilon$ iterations with step size $\epsilon=0.5$.

%%%%%%%%%%%%%%%%%%%%%%%%%%%%%%%%%%%%%%%%%%%%%%%%%%%%%%%%%%%%%%%%%%%%%%%%%%%%%%%
%%%%%%%%%%%%% IMPLEMENTATION DETAILS AND EXPERIMENTAL RESULTS %%%%%%%%%%%%%%%%%

\subsection{Implementation Details and Experimental Results}
\label{subsec:results}

To tackle the problem of the long free phase relaxation and speed-up
the simulations, we use 'persistent particles' for the latent variables to re-use the previous fixed point
configuration for a particular example as a starting point for the next free phase
relaxation on that example.
This means that for each training example in the dataset, we store the state of
the hidden layers at the end of the free phase,
and we use this to initialize the state of the network at the next epoch.
This method is similar in spirit to the PCD algorithm (Persistent Contrastive Divergence)
for sampling from other energy-based models like the Boltzmann
machine \citep{tieleman2008training}.

We find that it helps regularize the network if we choose the sign of $\beta$ at random in the second phase.
Note that the weight updates remain consistent thanks to the factor $1/\beta$ in the update rule
$\Delta W_{ij} \propto \frac{1}{\beta} \left( \rho \left( u_i^\beta \right) \rho \left( u_j^\beta \right) - \rho \left( u_i^0 \right) \rho \left( u_j^0 \right) \right)$.
Indeed, the left-derivative and the right-derivative of the function $\beta \mapsto \rho \left( u_i^\beta \right) \rho \left( u_j^\beta \right)$ at the point $\beta=0$ coincide.

Although the theory presented in this paper requires a unique learning rate for all synaptic weights,
in our experiments we need to choose different learning rates
for the weight matrices of different layers to make the algorithm work.
We do not have a clear explanation for this fact yet, but we believe that this is due
to the finite precision with which we approach the fixed points.
Indeed, the theory requires to be exactly at the fixed points, but in practice we minimize the energy function by numerical optimization,
using Eq.~\ref{eq:clipped-gradient-descent}. The precision with which we approach the fixed points depends on hyperparameters such as
the step size $\epsilon$ and the number of iterations $n_{\rm iter}$.

Let us denote by $h_0, h_1, \cdots, h_N$
the layers of the network (where $h_0 = \x$ and $h_N = y$)
and by $W_k$ the weight matrix between the layers $h_{k-1}$ and $h_k$.
We choose the learning rate $\alpha_k$ for $W_k$ so that the quantities
$\frac{\norm{\Delta W_k}}{\norm{W_k}}$ for $k=1,\cdots,N$ are approximately
the same in average (over training examples), where $\norm{\Delta W_k}$ represents the weight
change of $W_k$ after seeing a minibatch.

The hyperparameters chosen for each model are shown in Table \ref{table:hyperparameters}
and the results are shown in Figure \ref{fig:train_error}.
We initialize the weights according to the Glorot-Bengio
initialization~\citep{GlorotAISTATS2010-small}.
For efficiency of the experiments, we use minibatches of 20 training examples.

\begin{center}
$\begin{array}{|c|cc|c|c|cccc|}
\hline
	\hbox{Architecture} & \hbox{Iterations}      & \hbox{Iterations}     & \epsilon & \beta & \alpha_1 & \alpha_2 & \alpha_3 & \alpha_4 \\
	                    & (\hbox{first phase})   & (\hbox{second phase}) &    &     &          &          &          &       \\
\hline
	784-500-10          & 20  & 4 & 0.5 & 1.0 & 0.1   & 0.05  &       &       \\
	784-500-500-10      & 100 & 6 & 0.5 & 1.0 & 0.4   & 0.1   & 0.01  &       \\
	784-500-500-500-10  & 500 & 8 & 0.5 & 1.0 & 0.128 & 0.032 & 0.008 & 0.002 \\
\hline
\end{array}$
  \captionof{table}{Hyperparameters. The learning rate $\epsilon$ is used
    for iterative inference (Eq.~\ref{eq:clipped-gradient-descent}).
    $\beta$ is the value of the clamping factor in the second phase.
    $\alpha_k$ is the learning rate for updating the parameters in layer $k$.}
\label{table:hyperparameters}
\end{center}

%%%%%%%%%%%%%%%%%%%%%%%%%%%%%%%%%%%%%%%%%%%%%%%%%%%%%%%%%%%%%%%%%%%%%%%%%%%%%%%
%%%%%%%%%%%%%%%%%%%%%% CONCLUSION %%%%%%%%%%%%%%%%%%%%%%%%%%%%%%%%%%%%%%%%%%%%%
%%%%%%%%%%%%%%%%%%%%%%%%%%%%%%%%%%%%%%%%%%%%%%%%%%%%%%%%%%%%%%%%%%%%%%%%%%%%%%%

\section{Discussion, Looking Forward}

%Our algorithm shows that stochastic gradient descent on an objective function
%in a recurrent network with clamped inputs
%can be achieved with a brief second phase in which the
%output units are slightly driven towards their target.
%This relies on the idea that we only need to see how a small change in $s$
%would yield a small change in the prediction error.
%In the second phase, we only need to nudge $y$ in the direction that reduces prediction error
%and let this small perturbation propagate in the network.
%Our algorithm does not require a long relaxation in the second phase,
%like in other algorithms such as Boltzmann machines and recurrent back-propagation
%~\citep{Pineda87,Almeida87}.

From a biological perspective, a troubling issue in the Hopfield model is
the requirement of symmetric weights between the units.
Note that the units in our model
need not correspond exactly to actual neurons in the brain (it could be
groups of neurons in a cortical microcircuit, for example). It remains to
be shown how a form of symmetry could arise from the learning procedure
itself (for example from autoencoder-like unsupervised learning)
or if a different formulation could eliminate the symmetry requirement.
Encouraging cues come from the observation
that denoising autoencoders without tied weights often end up learning
symmetric weights~\citep{Vincent-JMLR-2010-small}.
Another encouraging piece of evidence, also linked to autoencoders, is the
theoretical result from~\citet{Arora-et-al-2015}, showing that the symmetric solution
minimizes the autoencoder reconstruction error between two successive layers of
rectifying (ReLU) units, suggesting that symmetry may arise as the result
of an additional objective function making successive layers form an autoencoder.
Also, \citet{Lillicrap-et-al-arxiv2014} show that the backpropagation algorithm
for feedforward nets also works when the feedback weights are random,
and that in this case the feedforward weight tend to 'align' with the feedback weights.

%Another issue is the requirement of two distinct phases.
%Although the two phases in our model perform the same kind of neural computations,
%the synaptic plasticity rules are different for the two phases.
%We proposed to interpret the learning rule as a a STDP-like update rule during the $\xi$-phase (section \ref{sec:stdp}),
%but still, the $0$-phase requires a different synaptic dynamics,
%and it is not obvious how the dynamics of $\beta$ and $\theta$ could be synchronized.
%Again, note that feedforward nets and Boltzmann machines suffer from the same problem.
%Feedforward nets require a forward pass without weight updates and a backward pass with weight updates.
%The learning rule for Boltzmann machines requires a anti-Hebbian update in the negative phase and a Hebbian update in the positive phase.

Another practical issue is that we would like to reduce
the negative impact of a lengthy relaxation to a fixed point, especially
in the free phase. A possibility is explored
by~\citet{Bengio-fastinference-arXiv2016} and was initially
discussed by~\citet{Salakhutdinov+Hinton-2009-small} in the context of
a stack of RBMs: by making each layer a good autoencoder, it is possible
to make this iterative inference converge quickly after an initial
feedforward phase, because the feedback paths ``agree'' with the states
already computed in the feedforward phase.

Regarding synaptic plasticity,
the proposed update formula can be contrasted with theoretical synaptic learning rules
which are based on the Hebbian product of pre- and postsynaptic activity,
such as the BCM rule~\citep{Bienenstock82,Intrator+Cooper-1992}.
The update proposed here is particular in that
it involves the temporal derivative of the postsynaptic activity, rather than
the actual level of postsynaptic activity.

Whereas our work focuses on a rate model of neurons, see~\citet{Feldman-2012}
for an overview of synaptic plasticity that goes
beyond spike timing and firing rate, including
synaptic cooperativity (nearby synapses on the
same dendritic subtree) and depolarization (due to multiple consecutive
pairings or spatial integration across nearby locations on the dendrite,
as well as the effect of the synapse's distance to the soma).
In addition, it would be interesting to study update rules
which depend on the statistics of triplets
or quadruplets of spikes timings, as in~\citet{Froemke+Dan-2002,Gjorgjievaa-et-al-2011}.
These effects are not considered here but future work should
consider them.

Another question is that of time-varying input. Although this work makes
back-propagation more plausible for the case of a static input, the brain
is a recurrent network with time-varying inputs, and back-propagation
through time seems even less plausible than static back-propagation. An
encouraging direction is that proposed by~\citet{Ollivier-et-al-arXiv2015,Tallec-arxiv2017},
which shows that computationally efficient estimators of the gradient can
be obtained using a forward method (online estimation of the gradient),
which avoids the need to store all past states in training sequences, at the
price of a noisy estimator of the gradient.

%%%%%%%%%%%%%%%%%%%%%%%%%%%%%%%%%%%%%%%%%%%%%%%%%%%%%%%%%%%%%%%%
%%%%%%%%%%%%%%%%%%%%%%%%%%%%%%% ACKNOWLEDGEMENTS %%%%%%%%%%%%%%%
%%%%%%%%%%%%%%%%%%%%%%%%%%%%%%%%%%%%%%%%%%%%%%%%%%%%%%%%%%%%%%%%

\section*{Acknowledgments}

The authors would like to thank Akram Erraqabi, Alex Lamb, Alexandre Thiery, Mihir Mongia, Samira Shabanian and
Asja Fischer for feedback and discussions, as well as NSERC, CIFAR, Samsung
and Canada Research Chairs for funding, and Compute Canada
for computing resources. We would also like to thank the developers of Theano
\footnote{http://deeplearning.net/software/theano/}, for developing such a
powerful tool for scientific computing.

%%%%%%%%%%%%%%%%%%%%%%%%%%%%%%%%%%%%%%%%%%%%%%%%%%%%%%%%%%%%%%%%
%%%%%%%%%%%%%%%%%%%%%%%%%%%%%%% BIBLIOGRAPHIE %%%%%%%%%%%%%%%%%%
%%%%%%%%%%%%%%%%%%%%%%%%%%%%%%%%%%%%%%%%%%%%%%%%%%%%%%%%%%%%%%%%

\bibliographystyle{natbib}
\bibliography{strings,ml,aigaion,biblio}

\begin{thebibliography}{}

\bibitem[Almeida(1987)Almeida]{Almeida87}
Almeida, L.~B. (1987).
\newblock A learning rule for asynchronous perceptrons with feedback in a
  combinatorial environment.
\newblock In M.~Caudill and C.~Butler, editors, {\em IEEE International
  Conference on Neural Networks\/}, volume~2, pages 609--618, San Diego 1987.
  IEEE, New York.

\bibitem[Arora {\em et~al.}(2015)Arora, Liang, and Ma]{Arora-et-al-2015}
Arora, S., Liang, Y., and Ma, T. (2015).
\newblock Why are deep nets reversible: a simple theory, with implications for
  training.
\newblock Technical report, arXiv:1511.05653.

\bibitem[Bengio and Fischer(2015)Bengio and Fischer]{Bengio-arxiv2015}
Bengio, Y. and Fischer, A. (2015).
\newblock Early inference in energy-based models approximates back-propagation.
\newblock Technical Report arXiv:1510.02777, Universite de Montreal.

\bibitem[Bengio {\em et~al.}(2015a)Bengio, Mesnard, Fischer, Zhang, and
  Wu]{Bengio-et-al-arxiv2015b}
Bengio, Y., Mesnard, T., Fischer, A., Zhang, S., and Wu, Y. (2015a).
\newblock {STDP} as presynaptic activity times rate of change of postsynaptic
  activity.
\newblock arXiv:1509.05936.

\bibitem[Bengio {\em et~al.}(2015b)Bengio, Lee, Bornschein, and
  Lin]{Bengio-et-al-arxiv2015}
Bengio, Y., Lee, D.-H., Bornschein, J., and Lin, Z. (2015b).
\newblock Towards biologically plausible deep learning.
\newblock arXiv:1502.04156.

\bibitem[Bengio {\em et~al.}(2016)Bengio, Scellier, Bilaniuk, Sacramento, and
  Senn]{Bengio-fastinference-arXiv2016}
Bengio, Y., Scellier, B., Bilaniuk, O., Sacramento, J., and Senn, W. (2016).
\newblock Feedforward initialization for fast inference of deep generative
  networks is biologically plausible.
\newblock {\em arXiv preprint arXiv:1606.01651\/}.

\bibitem[Bengio {\em et~al.}(2017)Bengio, Mesnard, Fischer, Zhang, and
  Wu]{Bengio-et-al-NECO2017}
Bengio, Y., Mesnard, T., Fischer, A., Zhang, S., and Wu, Y. (2017).
\newblock {STDP} as presynaptic activity times rate of change of postsynaptic
  activity approximates back-propagation.
\newblock {\em Neural Computation\/}, pages 1--23.

\bibitem[Berkes {\em et~al.}(2011)Berkes, Orban, Lengyel, and
  Fiser]{Berkes-et-al-2011}
Berkes, P., Orban, G., Lengyel, M., and Fiser, J. (2011).
\newblock Spontaneous cortical activity reveals hallmarks of an optimal
  internal model of the environment.
\newblock {\em Science\/}, {\bf 331}, 83--–87.

\bibitem[Bi and Poo(2001)Bi and Poo]{Bi+Poo-2001}
Bi, G. and Poo, M. (2001).
\newblock Synaptic modification by correlated activity: Hebb's postulate
  revisited.
\newblock {\em Annu. Rev. Neurosci.}, {\bf 24}, 139--–166.

\bibitem[Bienenstock {\em et~al.}(1982)Bienenstock, Cooper, and
  Munro]{Bienenstock82}
Bienenstock, E.~L., Cooper, L.~N., and Munro, P.~W. (1982).
\newblock Theory for the development of neuron selectivity: Orientation
  specificity and binocular interaction in visual cortex.
\newblock {\em Journal of Neuroscience\/}, {\bf 2}.

\bibitem[Feldman(2012)Feldman]{Feldman-2012}
Feldman, D.~E. (2012).
\newblock The spike timing dependence of plasticity.
\newblock {\em Neuron\/}, {\bf 75}(4), 556--571.

\bibitem[Friston and Stephan(2007)Friston and Stephan]{Friston+Stephan-2007}
Friston, K.~J. and Stephan, K.~E. (2007).
\newblock Free-energy and the brain.
\newblock {\em Synthese\/}, {\bf 159}, 417--–458.

\bibitem[Froemke and Dan(2002)Froemke and Dan]{Froemke+Dan-2002}
Froemke, R.~C. and Dan, Y. (2002).
\newblock Spike-timing-dependent synaptic modification induced by natural spike
  trains.
\newblock {\em Nature\/}, {\bf 416}(6879), 433--438.

\bibitem[Gerstner {\em et~al.}(1996)Gerstner, Kempter, van Hemmen, and
  Wagner]{Gerstner-et-al-1996}
Gerstner, W., Kempter, R., van Hemmen, J., and Wagner, H. (1996).
\newblock A neuronal learning rule for sub-millisecond temporal coding.
\newblock {\em Nature\/}, {\bf 386}, 76--78.

\bibitem[Gjorgjievaa {\em et~al.}(2011)Gjorgjievaa, Clopathb, Audetc, and
  Pfister]{Gjorgjievaa-et-al-2011}
Gjorgjievaa, J., Clopathb, C., Audetc, J., and Pfister, J.-P. (2011).
\newblock A triplet spike-timing–dependent plasticity model generalizes the
  bienenstock–cooper–munro rule to higher-order spatiotemporal
  correlations.
\newblock {\em {PNAS}\/}, {\bf 108}(48).

\bibitem[Glorot and Bengio(2010)Glorot and Bengio]{GlorotAISTATS2010-small}
Glorot, X. and Bengio, Y. (2010).
\newblock Understanding the difficulty of training deep feedforward neural
  networks.
\newblock In {\em AISTATS'2010\/}.

\bibitem[Hertz {\em et~al.}(1997)Hertz, Krogh, Lautrup, and Lehmann]{Hertz97}
Hertz, J.~A., Krogh, A., Lautrup, B., and Lehmann, T. (1997).
\newblock Nonlinear backpropagation: doing backpropagation without derivatives
  of the activation function.
\newblock {\em IEEE Transactions on neural networks\/}, {\bf 8}(6), 1321--1327.

\bibitem[Hinton(2002)Hinton]{Hinton2002}
Hinton, G.~E. (2002).
\newblock Training products of experts by minimizing contrastive divergence.
\newblock {\em Neural Computation\/}, {\bf 14}, 1771--1800.

\bibitem[Hinton and Sejnowski(1986)Hinton and Sejnowski]{hinton1986learning}
Hinton, G.~E. and Sejnowski, T.~J. (1986).
\newblock Learning and releaming in boltzmann machines.
\newblock {\em Parallel distributed processing: Explorations in the
  microstructure of cognition\/}, {\bf 1}, 282--317.

\bibitem[Hopfield(1984)Hopfield]{Hopfield84}
Hopfield, J.~J. (1984).
\newblock Neurons with graded responses have collective computational
  properties like those of two-state neurons.
\newblock {\em Proceedings of the National Academy of Sciences, USA\/}, {\bf
  81}.

\bibitem[Intrator and Cooper(1992)Intrator and Cooper]{Intrator+Cooper-1992}
Intrator, N. and Cooper, L.~N. (1992).
\newblock Objective function formulation of the {BCM} theory of visual cortical
  plasticity: statistical connections, stability conditions.
\newblock {\em Neural Networks\/}, {\bf 5}, 3--17.

\bibitem[LeCun and Cortes(1998)LeCun and Cortes]{lecun1998mnist}
LeCun, Y. and Cortes, C. (1998).
\newblock The mnist database of handwritten digits.

\bibitem[Lillicrap {\em et~al.}(2014)Lillicrap, Cownden, Tweed, and
  Akerman]{Lillicrap-et-al-arxiv2014}
Lillicrap, T.~P., Cownden, D., Tweed, D.~B., and Akerman, C.~J. (2014).
\newblock Random feedback weights support learning in deep neural networks.
\newblock arXiv:1411.0247.

\bibitem[Markram and Sakmann(1995)Markram and Sakmann]{Markram+Sakmann-1995}
Markram, H. and Sakmann, B. (1995).
\newblock Action potentials propagating back into dendrites triggers changes in
  efficacy.
\newblock {\em Soc. Neurosci. Abs\/}, {\bf 21}.

\bibitem[Markram {\em et~al.}(2012)Markram, Gerstner, and
  Sjöström]{Markram-et-al-2012}
Markram, H., Gerstner, W., and Sjöström, P. (2012).
\newblock Spike-timing-dependent plasticity: A comprehensive overview.
\newblock {\em Frontiers in synaptic plasticity\/}, {\bf 4}(2).

\bibitem[Mesnard {\em et~al.}(2016)Mesnard, Gerstner, and Brea]{Mesnard2016}
Mesnard, T., Gerstner, W., and Brea, J. (2016).
\newblock Towards deep learning with spiking neurons in energy based models
  with contrastive hebbian plasticity.
\newblock {\em arXiv preprint arXiv:1612.03214\/}.

\bibitem[Movellan(1990)Movellan]{Movellan-1990}
Movellan, J.~R. (1990).
\newblock Contrastive {H}ebbian learning in the continuous {H}opfield model.
\newblock In {\em Proc. 1990 Connectionist Models Summer School\/}.

\bibitem[Ollivier {\em et~al.}(2015)Ollivier, Tallec, and
  Charpiat]{Ollivier-et-al-arXiv2015}
Ollivier, Y., Tallec, C., and Charpiat, G. (2015).
\newblock Training recurrent networks online without backtracking.
\newblock Technical report, arXiv:1507.07680.

\bibitem[O'Reilly(1996)O'Reilly]{OReilly-1996}
O'Reilly, R.~C. (1996).
\newblock Biologically plausible error-driven learning using local activation
  differences: The generalized recirculation algorithm.
\newblock {\em Neural Computation\/}, {\bf 8}(5), 895--938.

\bibitem[Pineda(1987)Pineda]{Pineda87}
Pineda, F.~J. (1987).
\newblock Generalization of back-propagation to recurrent neural networks.
\newblock {\em Pattern Recognition Letters\/}, {\bf 59}, 2229--2232.

\bibitem[Salakhutdinov and Hinton(2009)Salakhutdinov and
  Hinton]{Salakhutdinov+Hinton-2009-small}
Salakhutdinov, R. and Hinton, G.~E. (2009).
\newblock Deep {Boltzmann} machines.
\newblock In {\em AISTATS'2009\/}, pages 448--455.

\bibitem[Sutskever and Tieleman(2010)Sutskever and
  Tieleman]{sutskever2010convergence}
Sutskever, I. and Tieleman, T. (2010).
\newblock {On the Convergence Properties of Contrastive Divergence}.
\newblock In Y.~W. Teh and M.~Titterington, editors, {\em Proc. of the
  International Conference on Artificial Intelligence and Statistics
  (AISTATS)\/}, volume~9, pages 789--795.

\bibitem[Tallec and Ollivier(2017)Tallec and Ollivier]{Tallec-arxiv2017}
Tallec, C. and Ollivier, Y. (2017).
\newblock Unbiased online recurrent optimization.
\newblock {\em arXiv preprint arXiv:1702.05043\/}.

\bibitem[Tieleman(2008)Tieleman]{tieleman2008training}
Tieleman, T. (2008).
\newblock Training restricted boltzmann machines using approximations to the
  likelihood gradient.
\newblock In {\em Proceedings of the 25th international conference on Machine
  learning\/}, pages 1064--1071. ACM.

\bibitem[Vincent {\em et~al.}(2010)Vincent, Larochelle, Lajoie, Bengio, and
  Manzagol]{Vincent-JMLR-2010-small}
Vincent, P., Larochelle, H., Lajoie, I., Bengio, Y., and Manzagol, P.-A.
  (2010).
\newblock Stacked denoising autoencoders: Learning useful representations in a
  deep network with a local denoising criterion.
\newblock {\em J. Machine Learning Res.}, {\bf 11}.

\bibitem[Xie and Seung(2000)Xie and Seung]{Xie+Seung-NIPS1999}
Xie, X. and Seung, H.~S. (2000).
\newblock Spike-based learning rules and stabilization of persistent neural
  activity.
\newblock In S.~Solla, T.~Leen, and K.~M\"{u}ller, editors, {\em Advances in
  Neural Information Processing Systems 12\/}, pages 199--208. MIT Press.

\bibitem[Xie and Seung(2003)Xie and Seung]{Xie+Seung-2003}
Xie, X. and Seung, H.~S. (2003).
\newblock Equivalence of backpropagation and contrastive {H}ebbian learning in
  a layered network.
\newblock {\em Neural Computation\/}.

\end{thebibliography}

%%%%%%%%%%%%%%%%%%%%%%%%%%%%%%%%%%%%%%%%%%%%%%%%%%%%%%%%%%%%%%%%
%%%%%%%%%%%%%%%%%%%%%%%%% APPENDIX %%%%%%%%%%%%%%%%%%%%%%%%%%%%%
%%%%%%%%%%%%%%%%%%%%%%%%%%%%%%%%%%%%%%%%%%%%%%%%%%%%%%%%%%%%%%%%

\appendix
\part*{Appendix}

%%%%%%%%%%%%%%%%%%%%%%%%%%%%%%%%%%%%%%%%%%%%%%%%%%%%%%%%%%%%%%%%%%%
%%%%%%%%%%%%%%%% PROOF OF THE THEOREM %%%%%%%%%%%%%%%%%%%%%%%%%%%%%

\section{Proof of the Gradient Formula (Theorem \ref{thm:deterministic})}
\label{appendix:theorem}

Here we prove Theorem \ref{thm:deterministic} by directly computing the gradient of $J$.
Another proof based on constrained optimization is proposed in Appendix \ref{appendix:constrained-optimization}.

We first state and prove a lemma for a twice differentiable function $F(\theta,\beta,s)$.
We assume that the conditions of the implicit function theorem are satisfied
so that the fixed point $s_\theta^\beta$ is a continuously differentiable function of $(\theta,\beta)$.
Since $\v$ does not play any role in the lemma, its dependence is omitted in the notations.

\begin{lem}[Deterministic version]
  \label{lemma:deterministic}
  Let $F(\theta,\beta,s)$ be a twice differentiable function and $s_\theta^\beta$ a fixed point characterized by
  \begin{equation}
    \label{eqn:fx-pt-eq}
    \frac{\partial F}{\partial s} \left( \theta,\beta,s_\theta^\beta \right) = 0.
  \end{equation}
  Then we have
  \begin{equation}
    \label{eq:thm-deterministic}
    \left( \frac{d}{d \theta} \frac{\partial F}{\partial \beta} \left( \theta,\beta,s_\theta^\beta \right) \right)^T
    = \frac{d}{d \beta} \frac{\partial F}{\partial \theta} \left( \theta,\beta,s_\theta^\beta \right).
  \end{equation}
  The notations $\frac{\partial F}{\partial \theta}$ and $\frac{\partial F}{\partial \beta}$
  are used to mean the \textit{partial derivatives} with respect to the first and second arguments of $F$ respectively,
  whereas $\frac{d}{d \theta}$ and $\frac{d}{d \beta}$ represent the \textit{total derivatives}
  with respect to $\theta$ and $\beta$ respectively
  (which include the differentiation path through $s_\theta^\beta$).
  The total derivative $\frac{d}{d \theta}$ (resp. $\frac{d}{d \beta}$) is performed for fixed $\beta$ (resp. fixed $\theta$).
\end{lem}

%Note that the object on both sides of Eq.~\ref{eq:thm-deterministic} is a matrix of size $\dim(\beta)\times\dim(\theta)$.
Interestingly, the variables $\theta$ and $\beta$ play symmetric roles in Eq.~\ref{eq:thm-deterministic}.

\begin{proof}[A Concise Proof of Lemma \ref{lemma:deterministic}]
	Consider the function
	\begin{equation}
	  \label{eq:G}
	  G(\theta,\beta) := F \left( \theta,\beta,s_\theta^\beta \right),
	\end{equation}
	which is the value of the total energy at the fixed point. The cross-derivatives of $G$ are transpose of each other:
	\begin{equation}
	  \left( \frac{\partial^2 G}{\partial \theta \partial \beta}(\theta,\beta) \right)^T
	  = \frac{\partial^2 G}{\partial \beta \partial \theta}(\theta,\beta).
	\end{equation}
	This can be rewritten in the form
	\begin{equation}
		\label{eq:cross-derivatives}
		\left( \frac{d}{d \theta} \frac{d}{d \beta} F \left( \theta,\beta,s_\theta^\beta \right) \right)^T
		= \frac{d}{d \beta} \frac{d}{d \theta} F \left( \theta,\beta,s_\theta^\beta \right).
	\end{equation}
	By the chain rule of differentiation we have
	\begin{equation}
	  \label{eq:derivative-beta}
	  \frac{d}{d\beta} F \left( \theta,\beta,s_\theta^\beta \right)
	  = \frac{\partial F}{\partial \beta} \left( \theta,\beta,s_\theta^\beta \right)
	  + \frac{\partial F}{\partial s} \left( \theta,\beta,s_\theta^\beta \right) \cdot \frac{\partial s_\theta^\beta}{\partial \beta}
	  = \frac{\partial F}{\partial \beta} \left( \theta,\beta,s_\theta^\beta \right).
	\end{equation}
	Here we have used the fixed point condition (Eq.~\ref{eqn:fx-pt-eq}).
	Similarly we have
	\begin{equation}
	  \label{eq:derivative-theta}
	  \frac{d}{d\theta} F \left( \theta,\beta,s_\theta^\beta \right)
	  = \frac{\partial F}{\partial \theta} \left( \theta,\beta,s_\theta^\beta \right).
	\end{equation}
	Plugging Eq.~\ref{eq:derivative-beta} and Eq.~\ref{eq:derivative-theta} in Eq.~\ref{eq:cross-derivatives}, we get
	\begin{equation}
		\left( \frac{d}{d \theta} \frac{\partial F}{\partial \beta} \left( \theta,\beta,s_\theta^\beta \right) \right)^T
		= \frac{d}{d \beta} \frac{\partial F}{\partial \theta} \left( \theta,\beta,s_\theta^\beta \right).
	\end{equation}
\end{proof}

To provide the reader with more details and more insights, we propose another proof of Lemma \ref{lemma:deterministic}
in which we explicitly compute the cross-derivatives of the function $G$ (defined in Eq.~\ref{eq:G}).

\begin{proof}[Another Proof of Lemma \ref{lemma:deterministic}]
	First we differentiate the fixed point equation Eq.~\ref{eqn:fx-pt-eq} with respect to $\beta$:
	\begin{equation}
	  \label{eqn:d-dbeta}
	  \frac{d}{d\beta} \; (\ref{eqn:fx-pt-eq}) \qquad \Rightarrow \qquad
	  \frac{\partial^2 F}{\partial s \partial \beta}(\theta,\beta,s_\theta^\beta)
	  + \frac{\partial^2 F}{\partial s^2}(\theta,\beta,s_\theta^\beta) \cdot \frac{\partial s_\theta^\beta}{\partial \beta}
	  = 0.
	\end{equation}
	Using again the chain rule of differentiation and Eq.~\ref{eqn:d-dbeta},
	the transpose of the left-hand side of Eq.~\ref{eq:thm-deterministic} can be rewritten
	\begin{align}
	  \frac{d}{d \theta} \frac{\partial F}{\partial \beta} \left( \theta,\beta,s_\theta^\beta \right)
	  & = \frac{\partial^2 F}{\partial \theta \partial \beta} \left( \theta,\beta,s_\theta^\beta \right)
	  + \left( \frac{\partial s_\theta^\beta}{\partial \theta} \right)^T \cdot \frac{\partial^2 F}{\partial s \partial \beta}(\theta,\beta,s_\theta^\beta) \\
	  & = \frac{\partial^2 F}{\partial \theta \partial \beta} \left( \theta,\beta,s_\theta^\beta \right)
	  - \left( \frac{\partial s_\theta^\beta}{\partial \theta} \right)^T \cdot \frac{\partial^2 F}{\partial s^2}(\theta,\beta,s_\theta^\beta) \cdot \frac{\partial s_\theta^\beta}{\partial \beta}.
	  \label{eqn:beta}
	\end{align}
	Similarly we differentiate the fixed point equation Eq.~\ref{eqn:fx-pt-eq} with respect to $\theta$:
	\begin{equation}
	  \label{eqn:d-dtheta}
	  \frac{d}{d\theta} \; (\ref{eqn:fx-pt-eq}) \qquad \Rightarrow \qquad
	  \frac{\partial^2 F}{\partial s \partial \theta}(\theta,\beta,s_\theta^\beta)
	  +
	  \frac{\partial^2 F}{\partial s^2}(\theta,\beta,s_\theta^\beta) \cdot \frac{\partial s_\theta^\beta}{\partial \theta}
	  = 0.
	\end{equation}
	and obtain the following form for the right-hand side of Eq.~\ref{eq:thm-deterministic}:
	\begin{align}
		\frac{d}{d \beta} \frac{\partial F}{\partial \theta}(\theta,\beta,s_\theta^\beta)
		& = \frac{\partial^2 F}{\partial \beta \partial \theta}(\theta,\beta,s_\theta^\beta)
	+ \left( \frac{\partial s_\theta^\beta}{\partial \beta} \right)^T \cdot \frac{\partial^2 F}{\partial s \partial \theta}(\theta,\beta,s_\theta^\beta) \\
		& = \frac{\partial^2 F}{\partial \beta \partial \theta}(\theta,\beta,s_\theta^\beta)
	- \left( \frac{\partial s_\theta^\beta}{\partial \beta} \right)^T \cdot \frac{\partial^2 F}{\partial s^2}(\theta,\beta,s_\theta^\beta) \cdot \frac{\partial s_\theta^\beta}{\partial \theta}.
		\label{eqn:theta}
	\end{align}
	Clearly Eq.~\ref{eqn:beta} is the transpose of Eq.~\ref{eqn:theta}. Hence the result.
\end{proof}

We have just proved that
\begin{equation}
  \frac{d}{d \theta} \frac{\partial F}{\partial \beta} \left( \theta,\beta,s_\theta^\beta \right)
  = \frac{\partial^2 F}{\partial \theta \partial \beta} \left( \theta,\beta,s_\theta^\beta \right)
  - \left( \frac{\partial s_\theta^\beta}{\partial \theta} \right)^T \cdot
  \frac{\partial^2 F}{\partial s^2} \left( \theta,\beta,s_\theta^\beta \right) \cdot \frac{\partial s_\theta^\beta}{\partial \beta}.
\end{equation}
It is worth mentioning that we can show with similar arguments that
\begin{equation}
  \frac{d}{d \theta} \frac{\partial F}{\partial \beta} \left( \theta,\beta,s_\theta^\beta \right)
  = \frac{\partial^2 F}{\partial \theta \partial \beta} \left( \theta,\beta,s_\theta^\beta \right)
  - \frac{\partial^2 F}{\partial \theta \partial s} \left( \theta,\beta,s_\theta^\beta \right) \cdot
  \left( \frac{\partial^2 F}{\partial s^2} \left( \theta,\beta,s_\theta^\beta \right) \right)^{-1}
  \cdot \frac{\partial^2 F}{\partial s \partial \beta} \left( \theta,\beta,s_\theta^\beta \right).
\end{equation}
Now we show that the gradient formula (Theorem \ref{thm:deterministic})
naturally arises from Lemma \ref{lemma:deterministic}.

\begin{proof}[Proof of Theorem \ref{thm:deterministic}]
	According to Lemma \ref{lemma:deterministic} we have
	\begin{equation}
		\label{eq:learn-alg}
		\left( \frac{d}{d\theta} \frac{\partial F}{\partial \beta} \left( \theta,\v,\beta,s_{\theta,\v}^\beta \right) \right)^T
		= \frac{d}{d \beta} \frac{\partial F}{\partial \theta} \left( \theta,\v,\beta,s_{\theta,\v}^\beta \right).
	\end{equation}
	Recall that the objective function is defined as
	\begin{equation}
		J(\theta,\v)
		= C \left( \theta,\v,s_{\theta,\v}^0 \right)
		= \frac{\partial F}{\partial \beta} \left( \theta,\v,0,s_{\theta,\v}^0 \right).
	\end{equation}
	Thus, for $\beta=0$, the left-hand side of Eq.~\ref{eq:learn-alg} represents the gradient of the objective function with respect to $\theta$:
	\begin{equation}
		\label{eq:LHS}
		\frac{\partial J}{\partial \theta}(\theta,\v).
	\end{equation}
	On the other hand, the right-hand side of Eq.~\ref{eq:learn-alg} represents the derivative of the function
	\begin{equation}
		\beta \mapsto \frac{\partial F}{\partial \theta} \left( \theta,\v,\beta,s_{\theta,\v}^\beta \right),
	\end{equation}
	which, for $\beta=0$, can be rewritten
	\begin{equation}
	\label{eq:RHS}
	\lim_{\beta \to 0} \frac{1}{\beta} \left( \frac{\partial F}{\partial \theta} \left( \theta,\v,\beta,s_{\theta,\v}^\beta \right) - \frac{\partial F}{\partial \theta} \left( \theta,\v,0,s_{\theta,\v}^0 \right) \right).
	\end{equation}
	Therefore, combining Eq.\ref{eq:LHS} and Eq.~\ref{eq:RHS} we get the desired result.
\end{proof}

Finally, we prove Proposition \ref{prop:deterministic}.

\begin{proof}[Proof of Proposition \ref{prop:deterministic}]
	As in Lemma \ref{lemma:deterministic}, we omit to write the dependence on $\v$ in the notations.
	Multiplying both sides of Eq.~\ref{eqn:d-dbeta} on the left by $\left( \frac{\partial s_\theta^\beta}{\partial \beta} \right)^T$, we get
	\begin{equation}
	\left( \frac{\partial s_\theta^\beta}{\partial \beta} \right)^T \cdot \frac{\partial^2 F}{\partial s \partial \beta}(\theta,\beta,s_\theta^\beta)
	= - \left( \frac{\partial s_\theta^\beta}{\partial \beta} \right)^T \cdot \frac{\partial^2 F}{\partial s^2}(\theta,\beta,s_\theta^\beta) \cdot \frac{\partial s_\theta^\beta}{\partial \beta} \leq 0.
	\end{equation}
	This inequality holds because $\frac{\partial^2 F}{\partial s^2} \left( \theta,\beta,s_\theta^\beta \right)$ is positive definite
	as $s_\theta^\beta$ is a local minimum of $F$.
	In particular, for $\beta=0$, we get
	\begin{equation}
		\label{eq:prop-det-ineq}
		\left( \left. \frac{\partial s_\theta^\beta}{\partial \beta} \right|_{\beta=0} \right)^T \cdot
		\frac{\partial C}{\partial s} \left(\theta,s_\theta^0 \right) \leq 0.
	\end{equation}
	Here we have used the fact that $C \left(\theta,s \right) = \frac{\partial F}{\partial \beta}\left(\theta,0,s \right)$
	when the value of $\beta$ is set to $0$.
	Using the chain rule, we see that the left-hand side of Eq.~\ref{eq:prop-det-ineq} is the derivative of
	\begin{equation}
		\beta \mapsto C \left( \theta,s_\theta^\beta \right)
	\end{equation}
	at the point $\beta = 0$.	
	Hence the result.
	%point $s_{\theta,\v}^\beta$ in the direction of $\delta^T \cdot \frac{\partial s_\theta^\beta}{\partial \beta}$.
	%Moreover, if equality in Eq.~\ref{eq:prop-det-ineq} holds then
	%$\frac{\partial s_\theta^\beta}{\partial \beta} \cdot \delta = 0$ and
	%$\frac{\partial^2 F}{\partial s \partial \beta}(\theta,\beta,s_\theta^\beta) \cdot \delta = 0$.
	%The latter condition means that $s_\theta^\beta$
	%is a critical point of Eq.~\ref{eq:prop-deterministic}.
\end{proof}

%%%%%%%%%%%%%%%%%%%%%%%%%%%%%%%%%%%%%%%%%%%%%%%%%%%%%%%%%%%%%%%%
%%%%%%%%%%%%% CONSTRAINED OPTIMIZATION PROBLEM %%%%%%%%%%%%%%%%%

\section{Reformulation of the Training Objective as a Constrained Optimization Problem}
\label{appendix:constrained-optimization}

Here we give another proof for the gradient formula (Theorem \ref{thm:deterministic}).
Considering $\v$ as fixed,
and regarding $\theta$ and $s$ as the free parameters,
we can frame the training objective (for a single training example $\v$) as the following constrained optimization problem:
\begin{align}
	\text{find}       & \quad \underset{\theta,s}{\arg \min} \; C(\theta,\v,s) \\
	\text{subject to} & \quad \frac{\partial E}{\partial s}(\theta,\v,s) = 0. \label{eq:constraint}
\end{align}
%where the cost function $C(\theta,\v,s)$ is given by Eq.~\ref{eq:cost-function}.
Note that in more conventional machine learning algorithms, one only optimizes $\theta$,
since the prediction is an {\em explicit} function of $\theta$.
Here on the contrary, in the context of constrained optimization,
the state $s$ is regarded as belonging to the set of free parameters that
we optimize because the prediction is an
{\em implicit} function of $\theta$ through the constraint Eq.~\ref{eq:constraint}.

As usual for constrained optimization problems, we introduce the Lagrangian
\begin{equation}
	\label{eq:L}
	L(\theta,s,\lambda) := C(\theta,\v,s) + \lambda \cdot \frac{\partial E}{\partial s}(\theta,\v,s)
\end{equation}
where $\lambda$ is the vector of Lagrange multipliers.
We have omitted the explicit dependence on the data point $\v$ in the notation $L(\theta,s,\lambda)$,
since this variable is considered fixed.
Starting from the current parameter $\theta$,
we first find $s^*$ and $\lambda^*$ such that
\begin{equation}
	\label{eq:dL-dlambda=0}
  \frac{\partial L}{\partial \lambda} \left( \theta,s^*,\lambda^* \right) = 0
\end{equation}
and
\begin{equation}
	\label{eq:dL-ds=0}
  \frac{\partial L}{\partial s} \left( \theta,s^*,\lambda^* \right) = 0,
\end{equation}
and then we do one step of gradient descent on $L$ with respect
to $\theta$, that is
\begin{equation}
  \label{eq:update-theta}
  \Delta \theta \propto - \frac{\partial L}{\partial \theta} \left( \theta,s^*,\lambda^* \right).
\end{equation}
The first condition (Eq.~\ref{eq:dL-dlambda=0}) gives
\begin{equation}
  \frac{\partial E}{\partial s} \left( \theta,\v,s^* \right) = 0 \quad \Rightarrow \quad s^* = s_{\theta,\v}^0.
\end{equation}
Thus $s^*$ is the free fixed point.
Injecting this into the second condition (Eq.~\ref{eq:dL-ds=0}) we get
\begin{equation}
	\label{eq:lambda-star}
	\frac{\partial C}{\partial s} \left( \theta,\v,s_{\theta,\v}^0 \right) + \lambda^* \cdot \frac{\partial^2 E}{\partial s^2} \left( \theta,\v,s_{\theta,\v}^0 \right) = 0.
\end{equation}

To solve this equation for $\lambda^*$, we introduce the total energy $F$
and use the definition of $C$ and $E$ in terms of $F$ (Eq.~\ref{eq:E-C-F}).
We get
\begin{equation}
	\label{eq:lambda-star-2}
  \frac{\partial^2 F}{\partial \beta \partial s} \left( \theta,\v,0,s_{\theta,\v}^0 \right) + \lambda^* \cdot \frac{\partial^2 F}{\partial s^2} \left( \theta,\v,0,s_{\theta,\v}^0 \right) = 0.
\end{equation}
Comparing Eq.~\ref{eq:lambda-star-2} and the transpose of Eq.~\ref{eqn:d-dbeta} (evaluated at the point $\beta=0$),
and using the fact that $\frac{\partial^2 F}{\partial s^2} \left( \theta,\v,0,s_{\theta,\v}^0 \right)$ is invertible
(it is positive definite since $s_{\theta,\v}^0$ is a local minimum of $s \mapsto F(\theta,\v,0,s)$),
we conclude that
\begin{equation}
  \lambda^* = \left( \left. \frac{\partial s_\theta^\beta}{\partial \beta} \right|_{\beta=0} \right)^T,
\end{equation}
which is the derivative of the fixed point with respect to $\beta$.
Finally, using the values of $s^*$ and $\lambda^*$, and rewritting the Lagrangian (Eq.~\ref{eq:L}) in the form
\begin{equation}
	L(\theta,s,\lambda) = \frac{\partial F}{\partial \beta}(\theta,\v,0,s) + \lambda \cdot \frac{\partial F}{\partial s}(\theta,\v,0,s),
\end{equation}
we can compute the gradient of the Lagrangian:
\begin{align}
  \frac{\partial L}{\partial \theta}(\theta,s^*,\lambda^*)
  & = \frac{\partial^2 F}{\partial \beta \partial \theta}(\theta,\v,0,s^*) + \lambda^* \cdot \frac{\partial^2 F}{\partial s \partial \theta}(\theta,\v,0,s^*) \\
  & = \frac{\partial^2 F}{\partial \beta \partial \theta}(\theta,\v,0,s_{\theta,\v}^0) + \left( \left. \frac{\partial s_\theta^\beta}{\partial \beta} \right|_{\beta=0} \right)^T \cdot \frac{\partial^2 F}{\partial s \partial \theta}(\theta,\v,0,s_{\theta,\v}^0) \\
  & = \left. \frac{d}{d\beta} \right|_{\beta=0} \frac{\partial F}{\partial \theta}(\theta,\v,\beta,s_{\theta,\v}^\beta).
\end{align}
Therefore Eq.~\ref{eq:update-theta} can be rewritten
\begin{equation}
  \Delta \theta \propto - \lim_{\beta \to 0} \frac{1}{\beta} \left( \frac{\partial F}{\partial \theta} \left( \theta,\v,\beta,s_{\theta,\v}^\beta \right)
  - \frac{\partial F}{\partial \theta} \left( \theta,\v,0,s_{\theta,\v}^0 \right) \right).
\end{equation}

%%%%%%%%%%%%%%%%%%%%%%%%%%%%%%%%%%%%%%%%%%%%%%%%%%%%%%%%%%%%%%%%
%%%%%%%%%%%%%%%%%%% STOCHASTIC FRAMEWORK %%%%%%%%%%%%%%%%%%%%%%%

\section{Stochastic Framework}
\label{appendix:stochastic-framework}

In this section we present a stochastic framework that naturally extends the deterministic
framework studied in the paper.
The analysis presented here could be the basis for a machine learning framework for spiking networks ~\citep{Mesnard2016}.

Rather than the deterministic dynamical system Eq.~\ref{eq:gradient-system},
a more likely dynamics would include some form of noise.
As suggested by ~\citet{Bengio-arxiv2015},
injecting Gaussian noise in the gradient system Eq.~\ref{eq:gradient-system}
leads to a Langevin dynamics, which we write as the following stochastic differential equation:
\begin{equation}
	\label{eq:sde}
	ds = -\frac{\partial F}{\partial s}(\theta,\v,\beta,s)dt + \sigma dB(t),
\end{equation}
where $B(t)$ is a standard Brownian motion of dimension $\dim(s)$.
In addition to the force $-\frac{\partial F}{\partial s}(\theta,\v,\beta,s)dt$,
the Brownian term $\sigma dB(t)$ models some form of noise in the network.
For fixed $\theta$, $\v$ and $\beta$, the Langevin dynamics Eq.~\ref{eq:sde} is known
to converge to the Boltzmann distribution with temperature $T=\frac{1}{2}\sigma^2$
(consequence of the Fokker-Planck equation, a.k.a. Kolmogorov forward equation).
For simplicity, here we assume that $\sigma=\sqrt{2}$, so that $T=1$.

Let us denote by $p_{\theta,\v}^\beta$ the Boltzmann distribution
corresponding to the energy function $F$. It is defined by
\begin{equation}
	\label{eq:beta-distribution}
  p_{\theta,\v}^\beta(s) := \frac{e^{-F(\theta,\v,\beta,s)}}{Z_{\theta,\v}^\beta},
\end{equation}
where $Z_{\theta,\v}^\beta$ is the partition function
\begin{equation}
	Z_{\theta,\v}^\beta(s) := \int e^{-F(\theta,\v,\beta,s)}ds.
\end{equation}

After running the dynamics Eq.~\ref{eq:sde} for long enough,
we can obtain a sample $s^0$ from the stationary distribution $p_{\theta,\v}^0$ in the free phase (with $\beta = 0$)
and similarly we can obtain a sample $s^\beta$ from $p_{\theta,\v}^\beta$ in the weakly-clamped phase (with $\beta > 0$).
Theorem \ref{thm2:stochastic} below will show that
$\frac{1}{\beta} \left( \frac{\partial F}{\partial \theta} \left( \theta,\v,\beta,s^\beta \right)
- \frac{\partial F}{\partial \theta} \left( \theta,\v,0,s^0 \right) \right)$
is an unbiased estimator of the gradient of the following objective function:
\begin{equation}
	\widetilde{J}(\theta,\v) := \E_{\theta,\v}^0 \left[ C(\theta,\v,s) \right].
\end{equation}
Here $\E_{\theta,\v}^\beta$ denotes the expectation over $s \sim p_{\theta,\v}^\beta(s)$
and $C$ is the cost function.

\begin{theo}[Stochastic version]
	\label{thm2:stochastic}
	\begin{equation}
		\frac{\partial \widetilde{J}}{\partial \theta} \left( \theta,\v \right)
		= \lim_{\beta \to 0} \frac{1}{\beta} \left( \E_{\theta,\v}^\beta \left[ \frac{\partial F}{\partial \theta} \left( \theta,\v,\beta,s \right) \right]
		- \E_{\theta,\v}^0 \left[ \frac{\partial F}{\partial \theta} \left( \theta,\v,0,s \right) \right] \right),
	\end{equation}
\end{theo}

Theorem \ref{thm2:stochastic} generalizes Theorem \ref{thm:deterministic} to the stochastic framework
and is a consequence of Lemma \ref{thm:stochastic} below
(which itself is a generalization of Lemma \ref{lemma:deterministic}).

\begin{lem}[Stochastic version]
	\label{thm:stochastic}
	Let $F(\theta,\beta,s)$ be a twice differentiable function and $p_\theta^\beta$ the Boltzmann distribution defined by
	\begin{equation}
		p_\theta^\beta(s) := \frac{e^{-F(\theta,\beta,s)}}{Z_\theta^\beta},
	\end{equation}
	where $Z_\theta^\beta$ is the partition function
	\begin{equation}
		Z_\theta^\beta(s) := \int e^{-F(\theta,\beta,s)},
	\end{equation}
	and $\E_\theta^\beta$ the expectation over $s \sim p_\theta^\beta(s)$. Then we have
	\begin{equation}
		\label{eq:thm-stochastic}
		  \left( \frac{d}{d \theta} \E_\theta^\beta \left[ \frac{\partial F}{\partial \beta}(\theta,\beta,s) \right] \right)^T
		= \frac{d}{d \beta} \E_\theta^\beta \left[ \frac{\partial F}{\partial \theta}(\theta,\beta,s) \right].
	\end{equation}
	As in Theorem \ref{lemma:deterministic}, the variables $\theta$ and $\beta$ play symmetric roles in Eq.~\ref{eq:thm-stochastic}.
\end{lem}

\begin{proof}[Proof of Lemma \ref{thm:stochastic}]
	The differentials of the log partition function are equal to
	\begin{equation}
		\label{eq:dlogZ-dbeta}
		\frac{d}{d \beta} \ln \left( Z_\theta^\beta \right)
		= - \E_\theta^\beta \left[ \frac{\partial F}{\partial \beta} \left( \theta, \beta, s \right) \right]
	\end{equation}
	and
	\begin{equation}
		\frac{d}{d \theta} \ln \left( Z_\theta^\beta \right)
		= - \E_\theta^\beta \left[ \frac{\partial F}{\partial \theta} \left( \theta, \beta, s \right) \right].
	\end{equation}
	Therefore
	\begin{equation}
		\left( \frac{d}{d\theta} \E_\theta^\beta \left[ \frac{\partial F}{\partial \beta} \left( \theta,\beta,s \right) \right] \right)^T
		= - \left( \frac{d}{d \theta} \frac{d}{d\beta} \ln \left( Z_\theta^\beta \right) \right)^T
		= - \frac{d}{d \beta} \frac{d}{d\theta} \ln \left( Z_\theta^\beta \right)
		= \frac{d}{d\beta} \E_\theta^\beta \left[ \frac{\partial F}{\partial \theta} \left( \theta,\beta,s \right) \right].
	\end{equation}
\end{proof}

Finally we state a result similar to Proposition \ref{prop:deterministic} in the stochastic framework,
which shows that for a small $\beta>0$ the 'nudged' distribution $p_{\theta,\v}^\beta$ is better than the 'free' distribution $p_{\theta,\v}^0$
in terms of expected cost.

\begin{prop}[Stochastic version]
	The derivative of the function
	\begin{equation}
		\label{eq:prop-sto}
		\beta \mapsto \E_{\theta,\v}^\beta \left[ C(\theta,\v,s) \right]
	\end{equation}
	at the point $\beta=0$ is non-positive.
\end{prop}

\begin{proof}
	The derivative of Eq.~\ref{eq:prop-sto} is
	\begin{equation}
		- \E_{\theta,\v}^0 \left[ \left( C(\theta,\v,s) \right)^2 \right]
		+ \left( \E_{\theta,\v}^0 \left[ C(\theta,\v,s) \right] \right)^2
		= - \text{Var}_{\theta,\v}^0 \left[ C(\theta,\v,s) \right]
		\leq 0,
	\end{equation}
	where $\text{Var}_{\theta,\v}^0$ represents the variance over $s \sim p_{\theta,\v}^0(s)$.
\end{proof}

\end{document}